\documentclass{article}
\usepackage{graphicx} % Required for inserting images
\usepackage{arxiv}

\usepackage{amssymb}
\usepackage{subfigure}
\usepackage{amsmath}
\usepackage{color} 
\usepackage{appendix}
\usepackage{algorithm}
\usepackage{algorithmic}
\usepackage{float}

\usepackage{booktabs}
\usepackage{multirow}
\usepackage{makecell}

\allowdisplaybreaks[4]

\usepackage{amsthm}
\newtheorem{assumption}{\bf Assumption}[section]
\newtheorem{thm}{\bf Theorem}[section]

\newtheorem{proposition}[thm]{\bf Proposition}
\newtheorem{remark}{Remark} 
\newtheorem{definition}[thm]{\bf Definition}

\title{DeepSPoC: a deep learning based sequential propagation of chaos}
\author{{Kai Du} \thanks{Shanghai Center for Mathematical Sciences, Fudan University, Shanghai, China. Email: kdu@fudan.edu.cn}\\
\And
{Yongle Xie} \thanks{Shanghai Center for Mathematical Sciences, Fudan University, Shanghai, China. Email: 24110840016@m.fudan.edu.cn}\\
\And
{Tao Zhou} \thanks{Institute of Computational Mathematics and Scientific/Engineering Computing, AMSS, Chinese Academy of Sciences, Beijing, China. Email: tzhou@lsec.cc.ac.cn. }\\
\And
{Yuancheng Zhou } \thanks{Mathematics \& Science College, Shanghai Normal University, Shanghai, China. Email: yczhou@shnu.edu.cn.}\\
\\
}

\date{July 2026}

\begin{document}

\maketitle
\begin{abstract}
Classical particle methods based on propagation of chaos (PoC) have been developed for solving mean-field stochastic differential equations and their associated nonlinear Fokker--Planck equations. 
However, direct PoC implementations are difficult to apply to high-dimensional problems because they require simulating and storing large numbers of interacting particles, often with high particle-particle interaction costs.
Motivated by these limitations, we build on the recently proposed sequential propagation of chaos (SPoC) framework, which replaces the fully interacting particle system in PoC with a sequential interaction mechanism. 
Based on this structure, we present DeepSPoC, a neural particle method that embeds a neural density representation into the sequential particle dynamics.
DeepSPoC simulates particles batch by batch, while the neural network represents the evolving empirical law and is substituted into the coefficients of the mean-field SDE, thereby replacing direct particle-particle interactions with particle-network interactions. 
In DeepSPoC, a recently developed normalizing flow model called KRnet is used to approximate the empirical measure of particles.
Compared with direct particle implementations, DeepSPoC substantially reduces memory consumption and evaluates interaction terms more efficiently, thereby improving scalability for high-dimensional problems. 
We apply DeepSPoC to a wide range of mean-field equations and verify its effectiveness and computational advantages.

% Sequential propagation of chaos (SPoC) is a recently developed tool for solving mean-field stochastic differential equations and their associated nonlinear Fokker-Planck equations. 
% Based on the theory of SPoC, we present a new method (DeepSPoC) that combines the interacting particle system of SPoC with deep learning. 
% A recently developed normalizing flow model called KRnet is used to approximate the empirical measure of particles in our algorithm.
% Compared with classical PoC-based particle methods, DeepSPoC computes particle interactions more efficiently and substantially reduces the memory required to store particle trajectories. 
% These two features make our method applicable to the computation of complex high-dimensional problems that require simulations of large particle systems.
% We apply our method to a wide range of different types of mean-field equation and verify its effectiveness and advantages. 
\end{abstract}

\section{Introduction}
A wide range of nonlinear partial differential equations (PDEs) in kinetic theory can be associated with the so-called mean-field stochastic differential equations (SDEs), which can be expressed in the following general form:
\begin{equation}\label{genneral_eq}
    dX_t=b(t,X_{t-},\mu_t)dt + \sigma(t,X_{t-},\mu_t)dZ_t,
\end{equation}
where $X_t\in\mathbb{R}^d$, $b$ is a $d$-dimensional vector field, $\sigma$ is a $d\times m$ matrix-valued function, $\mu_t$ denotes the distribution of $X_t$, and $Z_t$ is an $m$-dimensional L\'{e}vy process, which includes Poisson processes, compound Poisson processes, Brownian motion, and other stable processes as special cases. 
The distribution $\mu_t$ of the solution to the mean-field SDE is a weak solution of the associated nonlinear PDE in the distributional sense. 
These nonlinear PDEs, commonly referred to as nonlinear Fokker--Planck equations, describe the time evolution of the probability distribution of particles under the influence of external forces (including random noise). 
These PDEs have important applications in fluid dynamics, mean-field games, biology, and other fields (see the review paper~\cite{reviewofapplication}). 
However, most of them are difficult to solve numerically because of the nonlinearity induced by the mean-field term, together with other unfavorable features such as high dimensionality, non-locality, and possible degeneracy.
The relation between PDEs and SDEs motivates us to solve such problems from the perspective of the associated SDEs. 
For mean-field SDEs, the theory of propagation of chaos (PoC)~\cite{mckean1967propagation,sznitman1991topics} shows that a wide range of these equations can be viewed as mean-field limit equations of interacting particle systems. 
Their solutions can therefore be approximated by the empirical measure of the particles in the system. 
This insight has led to the development of various particle methods, including deterministic particle methods~\cite{carrillo2019blob,carrillo2016numerical}, stochastic particle methods~\cite{bossy1997stochastic,le2017particle,belaribi2011probabilistic,belaribi2013probabilistic}, and related acceleration techniques such as the random batch method~\cite{randombatch}. 
In summary, all the above methods approximate mean-field SDEs via particle systems, through which the solutions of the corresponding PDEs can be obtained.

However, despite their effectiveness, particle methods based on PoC still suffer from several limitations.
First, they require the simulation of a large number of interacting particles (i.e., the numerical solution of a large system of stochastic or ordinary differential equations representing the interacting particles), which results in high computational costs, typically scaling as $O(N^2)$ with respect to the number of particles if no additional acceleration techniques are employed. 
Second, constructing the solution obtained by particle methods requires storing the trajectories of all simulated particles, which leads to substantial memory requirements. 
These two difficulties become more pronounced for high-dimensional problems, where more particles must be simulated and stored to ensure accuracy.

% Recently, a new class of particle systems known as sequential propagation of chaos (SPoC)~\cite{du2023sequential} was also proposed to approximate the mean-field limit equation.
% Instead of employing the classical fully interacting PoC structure, SPoC replaces it with a sequential interaction mechanism.
% This sequential structure allows the particle system to be computed in a recursive particle-by-particle (or batch-by-batch) manner, so that new particles can be incorporated without altering the existing ones, leading to a progressive refinement of the approximation without pre-specifying the total particle number.

Recently, there has been a growing interest in using deep learning techniques to solve PDEs. 
Deep learning methods such as physics-informed neural networks (PINNs)~\cite{pinn}, Deep Ritz methods~\cite{deepritz}, weak adversarial networks (WANs)~\cite{zang2020weak}, and Deep BSDE methods~\cite{han2018solving} have been successfully applied to various PDE-related problems.
In addition, generative models such as generative adversarial networks~\cite{yang2020physics}, variational autoencoders~\cite{kingma2013auto}, and normalizing flows~\cite{papamakarios2021normalizing,rezende2015variational} are also widely used to solve forward and inverse PDE problems. 
These deep learning approaches have many advantages. 
First, they are promising numerical methods for high-dimensional problems because they are mesh-free. 
Second, they can fully utilize the computing resources of GPUs. 
However, despite the many advantages of deep learning methods in solving PDEs, they can still face challenges when applied to Fokker--Planck equations containing nonlinear and nonlocal terms, which often arise when the coefficients depend on the current distribution.
Given the strength of neural networks in high-dimensional approximation and scientific computing, it is natural to ask whether they can be integrated with particle methods to obtain a more scalable approach to solving high-dimensional nonlinear Fokker--Planck equations.

In this paper, we develop a neural-network-based particle method for solving the mean-field SDEs in~(\ref{genneral_eq}). Our method builds on the sequential propagation of chaos (SPoC) framework~\cite{du2023sequential}, a recent theoretical variant of PoC in which the classical fully interacting particle representation is replaced by a sequential one. SPoC provides a recursive particle construction in which new particles, generated individually or in batches, interact with the empirical law accumulated from previously generated particles. This sequential structure offers a theoretical basis for progressively refining particle approximations without recomputing the existing particles.
Motivated by this sequential structure, we propose DeepSPoC, which embeds a neural density representation into the SPoC particle dynamics. More specifically, instead of storing all particle trajectories and representing the evolving law directly by an empirical measure, DeepSPoC uses a neural network to approximate the particle distribution generated during the sequential simulation. In most experiments, we employ KRnet, a recently developed normalizing flow model specifically designed for distribution approximation~\cite{tang2020deep}; in a few cases, fully connected neural networks are also used. The learned density is then substituted into the coefficients of the mean-field SDE to generate the next particle batch. In this way, direct particle-particle interactions in conventional particle implementations are replaced by particle-network interactions, leading to reduced memory consumption and improved computational scalability. This replacement eliminates the need to explicitly evaluate all pairwise particle interactions, thereby avoiding the $O(N^2)$ interaction cost of conventional direct particle implementations. Consequently, for a fixed neural-network architecture and fixed batch size, the overall computational cost scales linearly with the number of generated particles.
The dependence of $b$ and $\sigma$ on the distribution $\mu_t$ in our framework can take various forms, including density-dependent forms, such as porous medium equations and fractional porous medium equations, convolution forms, such as Keller--Segel equations, and other types. Applications of our method to these different forms of mean-field equations are presented in Section~\ref{neumerical results}.
Overall, the main contributions of our method can be summarized as follows:
\begin{itemize}
\item[$\bullet$] By using a neural network to parameterize the empirical measure of particles, we avoid storing particle trajectories and thus significantly reduce memory consumption.
\item[$\bullet$] For a fixed neural-network architecture and fixed batch size, the computational cost of DeepSPoC scales linearly with the number of generated particles. This improved computational scalability, together with the mesh-free nature of neural networks, makes the method suitable for high-dimensional nonlinear Fokker--Planck equations.
\item[$\bullet$] We demonstrate the flexibility of DeepSPoC on several representative mean-field equations with different types of distribution dependence, including density-dependent diffusion, convolution-type interactions, and moment-based mean-field interactions, as well as extensions involving nonlocal fractional dynamics.
\end{itemize}

The remainder of this paper is structured as follows. 
Section 2 gives a brief introduction to nonlinear Fokker--Planck equations and the related propagation of chaos theory, which provides the background for the DeepSPoC algorithm. 
Section 3 presents the motivation and detailed implementation of the DeepSPoC algorithm.
In Section~4, we provide an a posteriori error estimate for the algorithm under simplified assumptions.
The numerical results for DeepSPoC applied to various types of equations are shown in Section 5.

\section{Problem setup and relevant theoretical background}
\label{setup and background}
In this section, we first give a brief introduction to the nonlinear Fokker--Planck equations that we aim to solve in this work.
We then review the propagation of chaos (PoC) framework and its sequential variant (SPoC), together with their associated particle methods, which serve as the basis for the proposed DeepSPoC algorithm.

\subsection{Nonlinear Fokker--Planck equations associated with mean-field SDEs}

The nonlinear Fokker–Planck equations associated with the stochastic differential equations (SDEs) in~(\ref{genneral_eq}) take the form
\begin{equation}\label{general_pde}
    \partial_t\mu_t(x)=\mathcal{L}^{\ast}(t,\mu_t)(\mu_t(x))
\end{equation}
where $x\in\mathbb{R}^d$ and $\mu_t$ denote the probability distribution of the random variable $X_t$ in~(\ref{genneral_eq}).
The nonlinear operator $\mathcal{L}^{\ast}(t,\mu_t)$, which depends on both $t$ and $\mu_t$, is the dual (adjoint) operator of the infinitesimal generator of the process in~(\ref{genneral_eq}).
It includes contributions from both the drift term and the diffusion term driven by the Lévy process $Z_t$ in~(\ref{genneral_eq}).
As a concrete example, if $Z_t$ is chosen as a Brownian motion $B_t$, the SDE becomes a McKean--Vlasov type equation, and the corresponding PDE becomes
\begin{equation}\label{nonlinear-FPE}
   	\partial_t\mu_t(x)=-\nabla_x\cdot (b(t,x,\mu_t)\mu_t(x))+\frac{1}{2}\sum_{i,j=1}^{d}\partial_{x_i}\partial_{x_j}(a_{ij}(t,x,\mu_t)\mu_t(x)),
\end{equation}
which is a second-order Fokker--Planck equation. 
Here, the diffusion matrix is given by 
$$a(t,x,\mu):=\sigma(t,x,\mu)\sigma(t,x,\mu)^{T}$$
where $a(t,x,\mu)\in\mathbb{R}^{d\times d}$ and $a_{ij}$ denote its $(i,j)$-entry. The integer $d$ denotes the dimension of $X_t$. 
To derive that $\mu_t$ in~(\ref{genneral_eq}) is a weak solution of the PDE~(\ref{general_pde}), one can choose a test function $\phi$, apply It\^o's formula to $\phi(X_t)$, and then take expectations on both sides to obtain the weak form of the nonlinear evolution equation~\eqref{general_pde}. 
For the general It\^o formula for SDEs driven by Lévy processes and the explicit form of the operator $\mathcal{L}^{\ast}(t,\mu_t)$, we refer the reader to~\cite{applebaum2009levy}.

\subsection{Propagation of chaos}

The connection between nonlinear Fokker--Planck equations and mean-field SDEs enables us to solve the PDEs by computing the distributions of the associated mean-field SDEs. 
However, trajectories of a mean-field SDE cannot be simulated directly because the equation involves mean-field terms and the distribution $\mu_t$ is unknown. 
A practical approach is to simulate the trajectories of $N$ particles simultaneously and use their empirical distribution as an approximation to $\mu_t$. 
This is precisely the idea underlying the propagation of chaos (PoC) theory, which suggests constructing an interacting particle system to approximate the mean-field SDE. 

An interacting particle system with $N$ particles $\{ X^{n,N}:n=1,2,\dots,N \}$ that approximates the mean-field SDE~(\ref{genneral_eq}) typically takes the form
\begin{align}  \label{poc-particle-system}
    \left\{
    \begin{aligned}
        dX_t^{n,N} & = b\big(t,X_{t-}^{n,N},\mu_t^N\big)\,dt
        + \sigma\big(t,X_{t-}^{n,N},\mu_t^N\big)\,dZ_t^n, \\
        \mu_t^N & = \frac{1}{N}\sum_{i=1}^N \delta_{X_t^{i,N}},
    \end{aligned}
    \right.
\end{align}
where $X_t^{n,N}\in\mathbb{R}^d$ and $Z_t^n$ are independent $m$-dimensional L\'evy processes, and the initial data $X_0^{n,N}\sim\mu_0$ are i.i.d.\ $\mathbb{R}^d$-valued random variables.  
Below we present two representative examples that frequently arise in applications.

\medskip
\noindent\textbf{(i) Interacting particles through a kernel.}  
A common case occurs when particles interact through a vector-valued kernel $K:\mathbb{R}^d\to\mathbb{R}^d$.  
In this case, the mean-field SDE often reads
\begin{equation*}
    dX_t = (K \ast \mu_t)(X_t)\,dt + \sqrt{2\sigma}\,dB_t,
\end{equation*}
Here $X_t\in\mathbb{R}^d$, $B_t$ is a $d$-dimensional Brownian motion, and $\sigma>0$ is a scalar diffusion coefficient.
The associated $N$-particle system is
\begin{equation*}
    dX_t^i =  \frac{1}{N}\sum_{j=1}^N K\big(X_t^i - X_t^j\big)\,dt 
    + \sqrt{2\sigma}\,dB_t^i, 
    \quad i=1,\dots,N.
\end{equation*}
Here $B_t^i$, $i=1,\dots,N$, are independent $d$-dimensional Brownian motions.
The PoC theory for such systems, including some cases with singular interaction kernels, is discussed in~\cite{serfaty2020mean,jabinwang2018quantitative,pocreview}.

\medskip
\noindent\textbf{(ii) Density-dependent interactions.}  
An important density-dependent example is the porous medium equation
\[
    \partial_t \rho = \Delta \rho^m,
\]
which corresponds to the density-dependent SDE
\[
    dX_t = \sqrt{2}\,\rho^{\frac{m-1}{2}}(t,X_t)\,dB_t.
\]
Here $\rho(t,\cdot)$ is the density of $X_t\in\mathbb{R}^d$, and $B_t$ is a $d$-dimensional Brownian motion.
In this case, the density function is needed, so the empirical measure $\mu^N_t$ of $N$ particles is replaced by the mollified empirical density 
\(\rho_{N,\epsilon} = \mu_t^N \ast f_\epsilon\), 
where $f_\epsilon$ is often chosen to be the Gaussian mollifier
\begin{equation}\label{gaussian-kernel}
    f_{\epsilon}(x)=\frac{1}{\left(2\pi\epsilon^2\right)^{d/2}}
    e^{-\frac{|x|^2}{2\epsilon^2}}.
\end{equation}
The corresponding $N$-particle system is then given by
\begin{equation}\label{poc particle system for pme}
    dX_t^i = \sqrt{2}\,\rho_{N,\epsilon}^{\frac{m-1}{2}}(t,X_t^i)\,dB_t^i,
    \quad i=1,\dots,N.
\end{equation}
A PoC result for such mollified particle systems can be found in~\cite{pocforpme}.

\medskip
\noindent
In general, PoC theory establishes that, for a broad class of mean-field SDEs~(\ref{genneral_eq}), the empirical measure $\mu_t^N$ converges to the distribution $\mu_t$ in~(\ref{genneral_eq}) in an appropriate sense as the number of particles $N \to \infty$.

\subsection{Sequential propagation of chaos}
Unlike the interacting particle system~\eqref{poc-particle-system} in the PoC framework, sequential propagation of chaos (SPoC) provides a different particle interaction mechanism: each particle interacts only with the preceding ones.  
This structure naturally leads to an \emph{iterative} particle system, allowing new particles to be added progressively.  
In contrast, PoC-based particle systems require the particle number $N$ to be fixed in advance, and increasing $N$ requires recomputing the entire system.  
SPoC thus offers the flexibility to refine the solution by continuously augmenting the particle set, which is particularly advantageous when a coarse or preliminary particle-based solution is already available but higher accuracy is required.

The SPoC interacting particle system $\{X^n\}_{n\ge 1}$ takes the form
\begin{align}\label{spoc-particle-system}
    \left\{
    \begin{aligned}
        dX_t^{n} &= b\big(t,X_{t-}^{n},\mu_t^{n-1}\big)\,dt
        + \sigma\big(t,X_{t-}^{n},\mu_t^{n-1}\big)\,dZ_t^n,\\
        \mu_t^n  &= \mu_t^{n-1}+\alpha_n\left( \delta_{X_t^n}-\mu_t^{n-1} \right),
    \end{aligned}
    \right.
\end{align}
where $\mu_t^0$ can be chosen as an arbitrary initial guess of the distribution, 
$\alpha_n$ denotes the update rate (typically decreasing and converging to $0$ as $n\to\infty$), 
and $Z_t^n$ are independent $m$-dimensional L\'evy processes, while $X_0^{n}\sim\mu_0$ are i.i.d.\ $\mathbb{R}^d$-valued random variables.  

It is shown in~\cite{du2023sequential} that, for McKean--Vlasov diffusions (where $Z_t$ is a Brownian motion), if the drift $b$ and diffusion matrix $\sigma$ satisfy a monotonicity condition and the update rate $\alpha_n$ is chosen appropriately, then $\mu_t^n$ converges to the distribution $\mu_t$ of the mean-field SDE~\eqref{genneral_eq} as $n\to\infty$, with the convergence rate quantified in the Wasserstein distance.

Beyond the particle-by-particle form~\eqref{spoc-particle-system}, a \emph{batch-by-batch} variant of SPoC was also proposed in~\cite{du2023sequential}.  
In this setting, the $n$-th batch contains $K$ particles $\{X^{i,n}:i=1,\dots,K\}$, and we denote the empirical measure of the trajectories in the $n$-th batch by
\begin{equation}\label{empirical measure of a batch of particles}
    \hat{\mu}^n := \frac{1}{K}\sum_{i=1}^K \delta_{X^{i,n}},
\end{equation} 
with its time marginal
\[
    \hat{\mu}^n_t := \frac{1}{K}\sum_{i=1}^K \delta_{X^{i,n}_t}.
\]
The batch-by-batch SPoC system is then given by
\begin{align}\label{spoc-particle-system-batchbybatch}
    \left\{
    \begin{aligned}
        dX_t^{i,n} &= b\big(t,X_{t-}^{i,n},\mu_t^{n-1}\big)\,dt
        + \sigma\big(t,X_{t-}^{i,n},\mu_t^{n-1}\big)\,dZ_t^{i,n}, 
        \quad i=1,\dots,K,\\
        \mu_t^n  &= \mu_t^{n-1}+\alpha_n\left( \hat{\mu}^n_t-\mu_t^{n-1} \right),
    \end{aligned}
    \right.
\end{align}
with $Z_t^{i,n}$ independent $m$-dimensional L\'evy processes and $X_0^{i,n}\sim \mu_0$ i.i.d.\ $\mathbb{R}^d$-valued random variables.  

The recursive update structure of the batch-by-batch formulation resembles mini-batch stochastic gradient descent in machine learning, where sampled data batches are used to incrementally update a model. This analogy motivates us to turn the sequential particle construction of SPoC into a neural particle-to-density learning procedure, which forms the basis of the DeepSPoC algorithm introduced in the next section.

\section{DeepSPoC algorithm}
\subsection{Particle system of DeepSPoC}
In this section, we introduce the DeepSPoC algorithm.
We begin with the DeepSPoC particle system for approximating the mean-field SDE~\eqref{genneral_eq}, which is based on the SPoC particle system~\eqref{spoc-particle-system-batchbybatch}:
\begin{align}\label{deepspoc-particle-system-batchbybatch}
    \left\{
    \begin{aligned}
            dX_t^{i,n}=& b\left(t,X_{t-}^{i,n},\rho_{NN,\boldsymbol{\theta}_{n-1}}(t,\cdot)\right)dt+\sigma\left(t,X_{t-}^{i,n},\rho_{NN,\boldsymbol{\theta}_{n-1}}(t,\cdot)\right)dZ_t^{i,n}, i=1,\dots ,K\\
        \rho_{NN,\boldsymbol{\theta}_{n}} =& F\left( \rho_{NN,\boldsymbol{\theta}_{n-1}}, \hat{\mu}^n, \alpha_n \right).
    \end{aligned}
    \right.
\end{align}
Compared with~\eqref{spoc-particle-system-batchbybatch}, DeepSPoC replaces the empirical measure in the SPoC update with a neural density representation.
The update operator $F$ acts on the network parameters and incorporates the information carried by each newly simulated particle batch into the current neural approximation of the law.
In this sense, DeepSPoC can be viewed as a neural analogue of the sequential law-updating mechanism of SPoC.
This formulation provides an algorithmic way to transfer the sequential law-updating idea of SPoC to a neural density representation.
The numerical experiments in Section~5 provide empirical evidence for the effectiveness of this construction.

% By selecting an appropriate neural network and carefully designing the update operator, we aim to enable the neural network to learn from the particles in each batch, rather than storing the positions of all particles as in the original approaches, and to converge, through the DeepSPoC particle system, to the distribution of the solution of the mean-field limit equation~\eqref{genneral_eq}.
% Based on our theoretical foundation, we believe that, for such a non-stationary learning process, where the distribution followed by the particles in each batch changes along with the neural network, the DeepSPoC system has the potential to converge to the mean-field limit equation. 

The key remaining component of the system is the update operator $F$, which determines how the neural density representation is updated after each particle batch.
More precisely, given the current neural network $\rho_{NN,\boldsymbol{\theta}_{n-1}}$, the empirical measure $\hat{\mu}^n$ generated by the newly simulated particle batch, and the update rate $\alpha_n$, the operator $F$ maps the parameter $\boldsymbol{\theta}_{n-1}$ to the updated parameter $\boldsymbol{\theta}_{n}$.
In our implementation, this update is realized through a gradient-based optimization step.
Specifically, we construct a loss function from the current neural density and the empirical measure of the new particle batch, compute its gradient with respect to the network parameters, and perform one optimizer step with learning rate $\alpha_n$.
The detailed realization of this procedure is described in the next subsection.

\subsection{Detailed implementation of DeepSPoC algorithm}\label{detailed explanation of DeepSPoC} 
The detailed implementation of the DeepSPoC algorithm is based on the particle system~\eqref{deepspoc-particle-system-batchbybatch}. 
According to~\eqref{deepspoc-particle-system-batchbybatch}, in the $n$-th iteration, we simulate a batch of particles by substituting the neural network into the SDE and then apply the operator $F$ to update the neural network; we refer to this process as one epoch of DeepSPoC. 
We now describe the detailed procedure for a complete DeepSPoC epoch in practice.

First, we substitute the current neural network $\rho_{NN,\boldsymbol{\theta}}$ into the SDE~\eqref{deepspoc-particle-system-batchbybatch} and use the Euler--Maruyama scheme for the numerical simulation. 
We divide the total time interval $[0,T]$ uniformly into $M$ subintervals, i.e., $0=t_0<t_1<\dots<t_M=T$, set the time step size as $\Delta t=T/M$, and then simulate a batch of particles $\{ X^i : i=1,\dots,K \}$ as follows. For convergence results of Euler schemes for L\'evy-driven SDEs, readers can refer to~\cite{eulerscheme}.
\begin{equation}\label{euler-scheme}
    X^i_{t_m}-X^i_{t_{m-1}}=b\left(t_{m-1},X^i_{t_{m-1}},\rho_{NN,\boldsymbol{\theta}}(t_{m-1},\cdot)\right)\Delta t+\sigma\left(t_{m-1},X^i_{t_{m-1}},\rho_{NN,\boldsymbol{\theta}}(t_{m-1},\cdot)\right)\Delta Z_t^i,
\end{equation}
with the initial points $\{X_0^i\}_{i=1}^{K}$ sampled independently from the initial distribution $\mu_0$. We then obtain the empirical measure of the particles at each discrete time point, i.e.,
$$
\hat{\mu}_{t_m}=\frac{1}{K}\sum_{i=1}^{K}\delta_{X_{t_m}^i}, m=0,\cdots,M,
$$
which can be viewed as the discrete version of $\hat{\mu}$. 
We point out that, when simulating particle trajectories, the computational cost of evaluating the interaction terms $b$ and $\sigma$ does not increase with the number of particles in the system.
Therefore, the computational cost of DeepSPoC increases linearly with the number of particles. 

Second, we make the neural network learn from the empirical measure $\hat{\mu}$, which defines the operator $F$. 
Once the update rate $\alpha>0$ is chosen, the implementation of the update operator $F$ can be divided into three steps:
\begin{itemize}
    \item[Step1] Construct a loss function $L\left(\rho_{NN,\boldsymbol{\theta}},\{\hat{\mu}_{t_m}\}_{m=0}^{M}\right)$ that can be computed from the current neural network and the empirical measure $\hat{\mu}$.
    \item[Step2] Compute the gradient of the loss function with respect to the neural network's parameters, that is, compute $\nabla_{\boldsymbol{\theta}}L\left(\rho_{NN,\boldsymbol{\theta}},\{\hat{\mu}_{t_m}\}_{m=0}^{M}\right)$. 
    \item[Step3] Execute one gradient descent step on the neural network parameters with learning rate $\alpha$.
\end{itemize}
The gradient descent step in the numerical experiments presented in this paper is carried out using the Adam optimizer.
The above procedure constitutes one complete epoch of the DeepSPoC algorithm. 
A general framework of DeepSPoC is shown in Fig.~\ref{network}.
In Appendix~\ref{app:connection}, we show that iteration~\eqref{deepspoc-particle-system-batchbybatch} is connected to the DeepSPoC training process under this gradient descent update.

We emphasize that DeepSPoC is not a direct implementation of the linear weighted-average update in the original SPoC system. Instead, it should be viewed as a neural analogue of the sequential law-updating structure of SPoC. In each epoch, the current neural density is used to generate a new batch of particles, and the resulting empirical distribution provides stochastic information for updating the neural density. The gradient step then moves the current neural density toward this newly generated empirical distribution. In this sense, the neural update plays a role analogous to incorporating new particle information into the current approximation, although it is not identical to the linear averaging rule in SPoC.

DeepSPoC should also be distinguished from an exact fixed-point iteration for the distribution map associated with the frozen-law SDE. Such a fixed-point method would require accurately solving the SDE with a fixed input law and then accurately fitting the resulting output law at each iteration. DeepSPoC is not designed to solve this exact fitting problem. Rather, in the spirit of SPoC, it performs incremental stochastic updates driven by finite particle batches. Therefore, the finite batch and the one-step gradient-based update in each epoch are intrinsic components of the proposed sequential neural learning dynamics, rather than approximations to an exact full-batch fixed-point solver.

%(flow chart)
\begin{figure}[!tb] 
        \centering
        \includegraphics[scale=0.57]{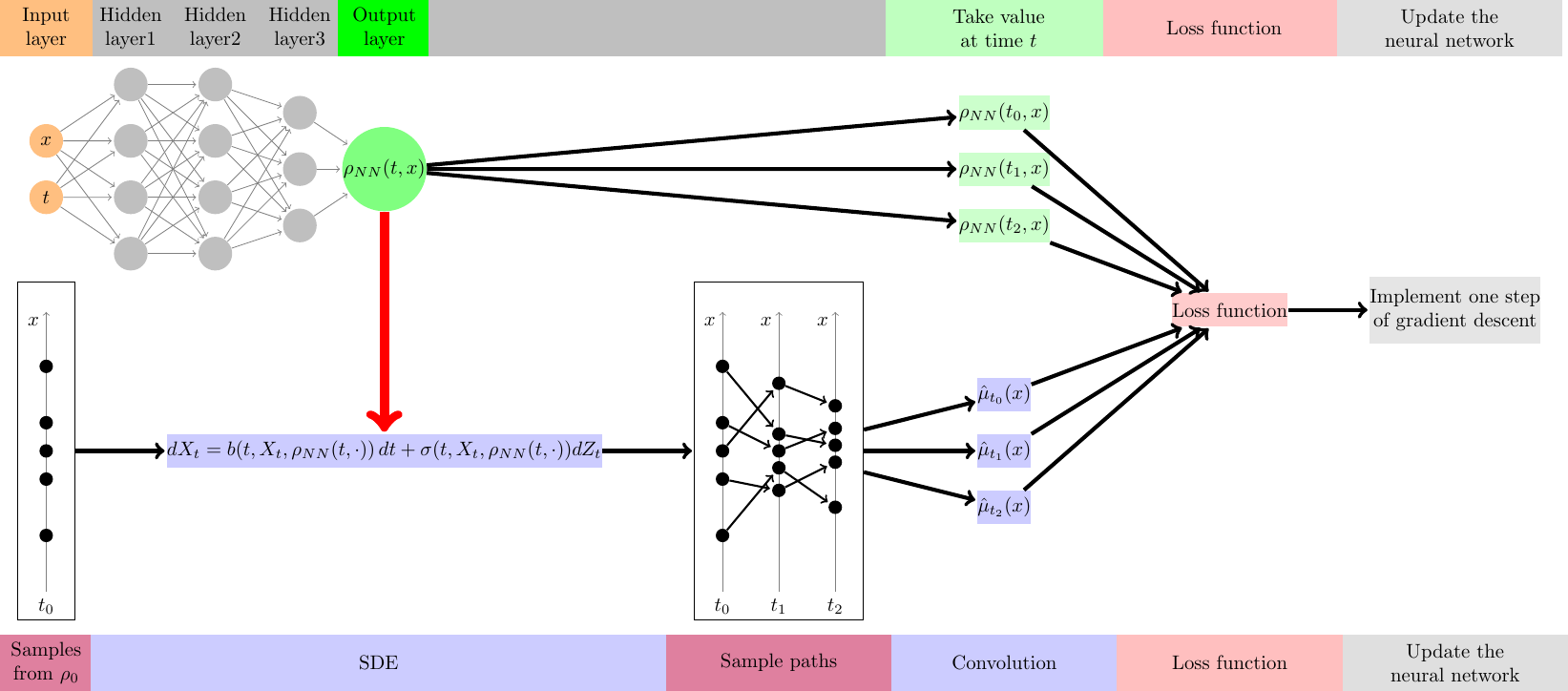}
        \caption{Flow chart of DeepSPoC}
        \label{network}
\end{figure}

In most of our numerical experiments, we use a recently developed temporal normalizing flow model called KRnet~\cite{tang2020deep,tang2022adaptive}, denoted by $\rho_{NF,\boldsymbol{\theta}}$.
Its structure will be introduced in the next subsection.
Normalizing flows~\cite{rezende2015variational,papamakarios2021normalizing} are a class of generative models that construct complex probability distributions by transforming a simple initial distribution, such as a Gaussian distribution, through a sequence of invertible and differentiable mappings. 
A temporal normalizing flow is a time-dependent normalizing flow that, at any fixed time $t$ in the temporal domain, reduces to a normalizing flow.
Owing to the inherent properties of normalizing flows, the model $\rho_{NF,\boldsymbol{\theta}}(t,\cdot)$ used as a density function automatically satisfies the following two conditions:
\begin{itemize}
    \item Nonnegativity, i.e., $\rho_{NF,\boldsymbol{\theta}}(t,x)\ge 0$, $\forall x\in \mathbb{R}^d$;
    \item Normalization, i.e., $\int\rho_{NF,\boldsymbol{\theta}}(t,x)\,dx=1$.
\end{itemize}
In contrast, when employing other models (such as fully connected networks) to approximate the density function, these two properties are not automatically guaranteed, which highlights a major advantage of using the KRnet model.

We use the following negative log-likelihood loss function:
\begin{align}\label{path loss}
    \begin{aligned}
    L\left(\rho_{NF,\boldsymbol{\theta}},\{\hat{\mu}_{t_m}\}_{\{m=0,\dots,M\}}\right)&= -\sum_{m=0}^{M}\int \text{log}(\rho_{NF,\boldsymbol{\theta}}(t_m,x))\hat{\mu}_{t_m}(dx)\\
&=-\sum_{m=0}^{M}\frac{1}{K}\sum_{i=1}^K\log(\rho_{NF,\boldsymbol{\theta}}(t_m,X_{t_m}^i)).
    \end{aligned}
\end{align}

The choice of the learning rate $\alpha_n$ for each epoch is also important for the convergence of the algorithm.
It represents the weight of the particles in the $n$-th epoch and plays a role similar to $\alpha_n$ in the SPoC particle system.
In practice, we find that reducing the learning rate by a contraction factor $\gamma<1$ every fixed number of epochs (denoted by $\Gamma$) can strike a balance between convergence speed and final accuracy; that is, for $n\ge 1$, we set
$$
\alpha_n = \alpha_0\gamma^{\lfloor \frac{n}{\Gamma} \rfloor},
$$
where $\lfloor x \rfloor$ denotes the floor of $x$.
If the problem is more complex and we want the neural network to learn from a larger number of particles before convergence, we can increase $\gamma$ and $\Gamma$; conversely, if the problem is simpler, we can decrease $\gamma$ and $\Gamma$.
However, we do not know exactly how far the result is from the true solution during the training process. 
An a posteriori error estimate is provided in Section~\ref{posterior}. Developing more general and convenient a posteriori error estimates for the error between the numerical solution and the true solution is left for future work.

The algorithm for DeepSPoC with KRnet is given in Algorithm~\ref{Deep SPoC Algorithm with path}.
\begin{algorithm}[!tb]
\caption{DeepSPoC algorithm}\label{Deep SPoC Algorithm with path}
\begin{algorithmic}[1]
\REQUIRE Initial distribution $\mu_0$, number of time intervals $M$ , number of sampled particles $K$ in each epoch, number of training epochs $N_{\text{epoch}}$, initial learning rate $\alpha_0$, contraction factor $\gamma$, step size $\Gamma$ and initial KRnet $\rho_{NF,\boldsymbol{\theta}}$.
\FOR{$\text{epoch}=1$ \TO $N_{\text{epoch}}$}
\STATE Decide the learning rate of this epoch $\alpha_{\text{epoch}}=\alpha_0\gamma^{\lfloor \frac{\text{epoch}}{\Gamma} \rfloor}$.
%\STATE Sample the training points $\left\{x_k\right\}_{k=0}^{k}$ from a uniform distribution on physics domain.
\STATE Independently generate starting point set $\{X_{t_0}^i\}_{\{i=1,\cdots,K\}}$ from $\mu_0$.
\FOR{$m=1$ \TO $M$}
\STATE Simulate $K$ trajectories of particles $\{X^i:i=1,\dots,K\}$ by Euler scheme
\begin{align*}
        X^i_{t_m}-X^i_{t_{m-1}}=&b\left(t_{m-1},X^i_{t_{m-1}},\rho_{NF,\boldsymbol{\theta}}(t_{m-1},\cdot)\right)\Delta t\\
        &+\sigma\left(t_{m-1},X^i_{t_{m-1}},\rho_{NF,\boldsymbol{\theta}}(t_{m-1},\cdot)\right)\Delta Z_t^i,
\end{align*}
and store the trajectories.
\ENDFOR
\STATE Build the loss function
$$
L(\boldsymbol{\theta})=-\sum_{m=0}^{M}\frac{1}{K}\sum_{i=1}^K\log(\rho_{NF,\boldsymbol{\theta}}(t_m,X_{t_m}^i))
$$
\STATE Calculate the gradient of the loss function with respect to $\boldsymbol{\theta}$ and implement one step of gradient descent by Adam optimizer with learning rate $\alpha_{\text{epoch}}$ to update $\boldsymbol{\theta}$.
\ENDFOR
\STATE Store and return the numerical solution.
\end{algorithmic}
\end{algorithm}

\begin{remark}\label{fully connected network method}
    Within the DeepSPoC framework, alternative neural network architectures and loss functions can also be employed.
    In our numerical experiments, we also use fully connected networks $\rho_{FC,\boldsymbol{\theta}}(t,x): \mathbb{R}^{d+1}\rightarrow \mathbb{R}$ to approximate the solution.
    The corresponding loss function $L$ is defined as follows:
\begin{equation}\label{Lsq}
L\left(\rho_{FC,\boldsymbol{\theta}},\{\hat{\mu}_{t_m}\}_{\{m=0,\dots,M\}}\right)=\sum_{m=0}^{M}\Vert \rho_{FC,\boldsymbol{\theta}}(t_m,\cdot)-\hat{\rho}_{t_m}(\cdot) \Vert_{L^2}^2,
\end{equation}
where $\hat{\rho}_{t_m}=\hat{\mu}_{t_m}\ast f_\epsilon$ is the mollified empirical density. 
In practice, the loss function is estimated by a Monte Carlo approximation, in which training points $S=\{x_i\}_{i=1}^{K_0}$ are uniformly sampled over the prescribed domain: 
$$
L\left(\rho_{FC,\boldsymbol{\theta}},\{\hat{\mu}_{t_m}\}_{\{m=0,\dots,M\}}\right)\approx \sum_{m=0}^{M}\frac{1}{K_0}\sum_{x\in S} \left|\rho_{FC,\boldsymbol{\theta}}(t_m,x)-\hat{\rho}_{t_m}(x) \right|^2.
$$
This scheme has demonstrated considerable effectiveness for relatively low-dimensional problems.
\end{remark}

\subsection{The structure of temporal KRnet $\rho_{NF,\boldsymbol{\theta}}$}
One difficulty in selecting the neural network in~\eqref{deepspoc-particle-system-batchbybatch} is maintaining nonnegativity and normalization. To overcome this difficulty, we use a recently developed temporal normalizing flow model, called (temporal) KRnet, to construct an effective density approximator for~\eqref{deepspoc-particle-system-batchbybatch}.

A normalizing flow learns an invertible transformation that acts as a transport map between a target distribution and a simple reference distribution. Consider a simple reference random variable $Z\in \mathbb{R}^d$ with a known probability density function $\rho_Z$, such as a standard Gaussian distribution. Let $f:x\rightarrow z$ be an invertible mapping defined by a normalizing flow. Then the PDF of $X=f^{-1}(Z)$ can be obtained by the change-of-variables formula, i.e.,
\begin{equation}
    \rho_{X}(x)=\rho_{Z}(f(x))\left| \det \nabla_x f(x)\right|,
\end{equation}
where $\nabla_x f(x)$ is the Jacobian matrix. The unknown invertible mapping can be learned from the observed data $X$ by maximum likelihood estimation. A temporal normalizing flow builds on the architecture of a standard normalizing flow. Let $\hat{x}$ be a real vector of dimension $d+1$ that includes the temporal variable, i.e., $\hat{x}=(t,x)$. Let $\hat{z}$ be a $(d+1)$-dimensional latent variable, i.e., $\hat{z}=(t^{\star},z)$, following a simple distribution. If the latent time $t^{\star}$ is exactly equal to the real time $t$, the PDF of the temporal normalizing flow can be written as
\begin{equation}
    \rho_{\hat{X}}(t,x)=\rho_{\hat{Z}}(t,z)\left| \frac{\partial z(t,x)}{\partial x}\right|, z=f(t,x).
\end{equation}
A standard way to construct a complex bijection $f$ is to hierarchically combine multiple simple bijective layers, each implemented as a shallow neural network, resulting in an overall deep neural network mapping. Our approach is based on a simplified time-dependent KRnet architecture~\cite{tang2020deep,tang2022adaptive}, which consists of actnorm layers and affine coupling layers with an invertible block-triangular structure.

\subsubsection{Actnorm layer: scale and bias layer}
In this layer, the following batch normalization, defined in~\cite{kingma2018glow}, is implemented:
$$
x_{[i+1]}=a\odot x_{[i]}+b, 
$$
where $a$ and $b$ are trainable parameters. The parameters $a$ and $b$ are initialized using the empirical mean and standard deviation of the data. The Actnorm layer helps stabilize and accelerate the training process, especially for complex and deep networks.

\subsubsection{Time-dependent  affine coupling layer}
Consider a partition $x_{[i]}=(x_{[i],1},x_{[i],2})$ with $x_{[i],1}\in \mathbb{R}^m$ and $x_{[i],2}\in \mathbb{R}^{d-m}$. The time-dependent affine coupling layer is defined as
$$
\begin{aligned}
&z_1=x_{[i],1},\\
&z_2=x_{[i],2}\odot (1_{d-m}+\alpha\tanh(s_t(t,x_{[i],1}))+e^{\beta}\odot \tanh(q_t(t,x_{[i],1})),
\end{aligned}
$$
where $|\alpha|<1$ is a tunable parameter (we recommend $\alpha=0.6$). The functions $s_t$ and $q_t$ are scaling and translation functions at time $t$, respectively, and define two maps from $\mathbb{R}^m$ to $\mathbb{R}^{d-m}$ that depend only on $x_{[i],1}$ and $t$. Since the functions $s_t$ and $q_t$ can be arbitrarily complicated, we model them using a neural network:
$$
(s_t,q_t)=\text{NN}(t,x_{[i],1}).
$$
Here, $\beta\in \mathbb{R}^n$ is a trainable variable. The inverse of the time-dependent affine coupling layer can be found in~\cite{zeng2023adaptive}.

\subsubsection{The main structure of temporal KRnet}
The main structure of temporal KRnet is defined by two loops: an outer loop $f_{[k]}^{\text{outer}}$ and an inner loop $f_{[k,i]}^{\text{inner}}$, where $k=1,\cdots,K-1$ and $i=1,\cdots,L$. KRnet can be written as
$$
z=f(t,x)=L_{N}\circ f_{[K-1]}^{\text{outer}}\circ\cdots\circ f_{[1]}^{\text{outer}}(t,x)
$$
where each outer loop $f_{[k]}^{\text{outer}}$ consists of a sequence of inner loops $f_{[k,l]}^{\text{inner}}$, $l=1,\cdots,L$, i.e.,
$$
f_{[k]}^{\text{outer}}=f_{[k,L]}^{\text{inner}}\circ \cdots \circ f_{[k,1]}^{\text{inner}}.
$$
Each $f_{[k,i]}^{\text{inner}}$ includes an actnorm layer $L_{[k,i]}^{\text{act}}$, a random orthogonal transformation $Q_{[k,i]}$, and a time-dependent affine coupling layer $L_{[k,i]}^{\text{aff}}$, i.e.,
$$
f_{[k,i]}^{\text{inner}}=L_{[k,i]}^{\text{act}}\circ Q_{[k,i]}\circ L_{[k,i]}^{\text{aff}}.
$$
Instead of using rotation layers in the inner loop, we add a random orthogonal transformation to each general coupling layer $f_{[k,i]}^{\text{inner}}$. This helps avoid dimensions with large approximation errors and therefore improves the overall accuracy of temporal KRnet. If all trainable variables in temporal KRnet are denoted by the parameter vector $\theta$, then the model can be represented as $\rho_{\theta}(t,x)$.

\section{A posteriori error estimation for DeepSPoC}\label{posterior}
It is of significant practical value to assess whether the numerical solution is close to the true solution during the training process.
In this section, we discuss a posteriori error estimation for the DeepSPoC algorithm, which provides upper and lower bounds on the distance between the numerical solution and the true solution under the metric defined below.

For many deep learning--based algorithms for solving PDEs, the value of the loss function can serve as an a posteriori error indicator.
For example, physics-informed neural networks (PINNs) usually use the training error, i.e., the residual obtained by substituting the neural-network-fitted solution into the model equation, to estimate the generalization error; see, for instance,~\cite{mishra2022estimates,mishra2023estimates,hillebrecht2022estimation}. 
For the DeepSPoC algorithm, however, the training loss cannot be used to estimate the error because it does not converge to zero; therefore, we need to explore other quantities for a posteriori error estimation. 

We first introduce the notation used in this section. 
We denote by $\Vert\cdot\Vert$ the Frobenius norm of a matrix, by $\mathcal{P}_p$ the space of all Borel probability measures $\mu$ on $\mathbb{R}^d$ such that $\Vert \mu \Vert_p:= (\int_{\mathbb{R}^d} |x|^{p}\mu(dx))^{1/p}< \infty$, and by $W_p(\mu, \nu)$ the $p$-Wasserstein distance for $\mu,\nu\in \mathcal{P}_p$, defined as 
$$
W_p(\mu,\nu)=\inf\{ (\mathbb{E}|X-Y|^p)^{1/p}: X\sim \mu ,Y\sim \nu \}.
$$

Similar to the case of PINNs, we can substitute the neural-network-fitted density into the mean-field SDE and solve the resulting equation. At this point, the SDE becomes Markovian and can therefore be simulated easily. By comparing the distance between the distribution of this solution and the neural-network-fitted distribution used as input, we obtain an a posteriori error estimate.

To rigorously state and prove the result, we need some assumptions about the mean-field SDE. 
First, we restrict the diffusion term to the Brownian-motion-driven case, where the SDE takes the following form:
\begin{equation}\label{mean field equation}
    dX_t=b(t,X_t,\text{Law}(X_t))dt+\sigma(t,X_t,\text{Law}(X_t))dB_t,
\end{equation}
where $\text{Law}(X_t)$ denotes the distribution of $X_t$, and the initial distribution is given by $X_0\sim \mu_0\in \mathcal{P}_2$. Furthermore, $b$ and $\sigma$ in the above equation are assumed to satisfy the following condition:
\begin{assumption}\label{assumption2}
    There exists a constant $C\ge 0 $ such that 
    \begin{align*}
    \begin{aligned}
        |b(t,x,\mu)-b(t,y,\nu)|&+\Vert \sigma(t,x,\mu)-\sigma(t,x,\nu) \Vert\le C(|x-y|+W_2(\mu,\nu)),\\
        &|b(t,0,\mu)|+\Vert \sigma(t,0,\mu) \Vert\le C(1+\Vert \mu \Vert_2)
    \end{aligned}
\end{align*}
for all $t\in [0,T]$, $x,y\in \mathbb{R}^d$, $\mu,\nu\in \mathcal{P}_2$.
\end{assumption}
% \begin{remark}
%     From the basic property of Wasserstein distance we know that $W_1(\mu,\nu)\le W_2(\mu,\nu)$ and $\Vert\mu \Vert_1\le \Vert\mu \Vert_2$, therefore Assumption \ref{assumption2} is actually weaker than Assumption \ref{assumption} in the last subsection.
% \end{remark}

In this section, we view a time-dependent distribution as a map from $[0,T]$ to $\mathcal{P}_2$ and denote it by $\mu_{\cdot}$. 
We denote by $\mathcal{P}_{2,\infty}([0,T])$ the space consisting of all such $\mathcal{P}_2$-valued maps $\mu_{\cdot}: [0,T]\rightarrow \mathcal{P}_2$ that are continuous with respect to the $W_2$-distance.
Specifically, $\mu_{\cdot}\in \mathcal{P}_{2,\infty}([0,T])$ satisfies $\sup_{0\le t\le T}\Vert \mu_t \Vert_2< \infty$. 
Our numerical solution $\rho_{FC,\boldsymbol{\theta}}(t,x)$ can also be considered in this space. 
We further introduce the following metric on $\mathcal{P}_{2,\infty}([0,T])$:
\begin{definition}
    For $\mu_{\cdot},\nu_{\cdot}\in \mathcal{P}_{2,\infty}([0,T])$ , $\alpha>0$, define
    $$
    H_{\alpha}(\mu_{\cdot},\nu_{\cdot}):= \left(\int_0^T e^{-\alpha t}W_2^2(\mu_t,\nu_t) dt\right)^{\frac{1}{2}}.
    $$
\end{definition}
It is easy to verify that $H_\alpha$ is a metric in $\mathcal{P}_{2,\infty}([0,T])$; in particular, the triangle inequality follows from
\begin{align*}
    H_{\alpha}(\mu^1_{\cdot},\mu^2_{\cdot})&\le \Big(H^2_{\alpha}(\mu^1_{\cdot},\mu^3_{\cdot})+H^2_{\alpha}(\mu^2_{\cdot},\mu^3_{\cdot})+2\int_0^Te^{-\alpha t}W_2(\mu_t^1,\mu_t^3)W_2(\mu_t^2,\mu_t^3)dt\Big)^{\frac{1}{2}}\\
    &\le \Big(H^2_{\alpha}(\mu^1_{\cdot},\mu^3_{\cdot})+ H^2_{\alpha}(\mu^2_{\cdot},\mu^3_{\cdot})+2 H_{\alpha}(\mu^1_{\cdot},\mu^3_{\cdot})H_{\alpha}(\mu^2_{\cdot},\mu^3_{\cdot})\Big)^{\frac{1}{2}}\\
    &=H_{\alpha}(\mu^1_{\cdot},\mu^3_{\cdot})+H_{\alpha}(\mu^2_{\cdot},\mu^3_{\cdot})
\end{align*}

Given $\mu_{\cdot}\in \mathcal{P}_{2,\infty}([0,T])$, the Markovian SDE
$$
dX_t=b(t,X_t,\mu_t)dt+\sigma(t,X_t,\mu_t)dB_t,
\qquad X_0\sim \mu_0,
$$
has a strong solution $(X_t)_{0\le t\le T}$ such that, under Assumption~\ref{assumption2}, the map
$t\in [0,T]\mapsto \operatorname{Law}(X_t)$, denoted by $\Phi(\mu_{\cdot})$, belongs to
$\mathcal{P}_{2,\infty}([0,T])$. Thus, we have defined a map
$$
\Phi:\mathcal{P}_{2,\infty}([0,T])\rightarrow \mathcal{P}_{2,\infty}([0,T]).
$$
Let
$$
C_0=2(T+1)C^2,
$$
where $C$ is the constant in Assumption~\ref{assumption2}. For $\alpha>C_0$, we further define
$$
\kappa_\alpha:=\left(\frac{C_0}{\alpha-C_0}\right)^{1/2}.
$$
The next proposition shows that, when $\alpha$ is sufficiently large, $\Phi$ is a contraction on
$\mathcal{P}_{2,\infty}([0,T])$ equipped with the metric $H_\alpha$.

\begin{proposition}\label{contraction proposition}
    Suppose that Assumption~\ref{assumption2} holds and that the initial distribution satisfies $\mu_0\in \mathcal{P}_2$. Assume that $\alpha>C_0$. Then, for any $\mu_{\cdot},\nu_{\cdot}\in \mathcal{P}_{2,\infty}([0,T])$, we have
    $$
    H_{\alpha}(\Phi(\mu_{\cdot}),\Phi(\nu_{\cdot}))\le \kappa_{\alpha}H_{\alpha}(\mu_{\cdot},\nu_{\cdot}).
    $$
\end{proposition}
The proof of this proposition is given in the appendix.

Furthermore, under Assumption~\ref{assumption2}, the distribution of the solution to~(\ref{mean field equation}) can also be regarded as an element of $\mathcal{P}_{2,\infty}([0,T])$, which we denote by $\mu^{\ast}_{\cdot}$. 
It is easy to see that $\mu^{\ast}_{\cdot}$ is a fixed point of $\Phi$. Using the contraction property of $\Phi$ when $\alpha>2C_0$, we obtain the following a posteriori error estimate:
\begin{thm}\label{thm:posterior}
    Suppose that Assumption~\ref{assumption2} holds and that the initial distribution satisfies $\mu_0 \in \mathcal{P}_2$. Let $\mu^{\ast}_{\cdot}$ be the distribution-valued function over $[0,T]$ corresponding to the solution of~(\ref{mean field equation}), and let $\alpha>2C_0$. Then, for any $\mu_{\cdot}\in \mathcal{P}_{2,\infty}([0,T])$, we have the following estimate:
    $$
   ( 
1+\kappa_{\alpha} )^{-1} H_{\alpha}(\mu_{\cdot}, \Phi(\mu_{\cdot})) \le H_{\alpha}(\mu_{\cdot},\mu^{\ast}_{\cdot})\le ( 
1-\kappa_{\alpha} )^{-1} H_{\alpha}(\mu_{\cdot}, \Phi(\mu_{\cdot})).
    $$
\end{thm}
\begin{proof}
    By Proposition~\ref{contraction proposition} and the triangle inequality for the metric $H_\alpha$, we have
    \begin{align*}
    \begin{aligned}
    H_\alpha(\mu_\cdot,\mu^\ast_\cdot)\le & H_\alpha(\mu_\cdot,\Phi(\mu_{\cdot}))+H_\alpha(\Phi(\mu_\cdot),\mu^{\ast}_{\cdot})\\
    = & H_\alpha(\mu_\cdot,\Phi(\mu_{\cdot}))+H_\alpha(\Phi(\mu_\cdot),\Phi(\mu^{\ast}_{\cdot}))\\
    \le & H_\alpha(\mu_\cdot,\Phi(\mu_{\cdot}))+\kappa_{\alpha}H_\alpha(\mu_\cdot,\mu^{\ast}_{\cdot}),
    \end{aligned}
\end{align*}
Since $\alpha>2C_0$, we have $\kappa_{\alpha}<1$, and therefore
$$
H_{\alpha}(\mu_{\cdot},\mu^{\ast}_{\cdot})\le ( 
1-\kappa_{\alpha} )^{-1} H_{\alpha}(\mu_{\cdot}, \Phi(\mu_{\cdot})).
$$
The left-hand inequality follows from
$$
H_{\alpha}(\mu_{\cdot},\Phi(\mu_{\cdot}))\le H_{\alpha}(\Phi(\mu_{\cdot}),\mu^{\ast}_{\cdot})+H_{\alpha}(\mu_{\cdot},\mu^{\ast}_{\cdot})\le ( 
1+\kappa_{\alpha})H_{\alpha}(\mu_{\cdot},\mu^{\ast}_{\cdot}).
$$
\end{proof}
For practical purposes, it is often more convenient to use the unweighted metric
$$
\mathcal{D}(\mu_\cdot,\nu_\cdot)
:=
\left(\int_0^T W_2^2(\mu_t,\nu_t),dt\right)^{1/2}.
$$
This metric is equivalent to $H_\alpha$ on the finite time interval $[0,T]$. Indeed, for any $\alpha>0$,
\begin{equation}\label{equivalence}
   e^{-\alpha T/2}\mathcal{D}(\mu_\cdot,\nu_\cdot)
\leq
H_\alpha(\mu_\cdot,\nu_\cdot)
\leq
\mathcal{D}(\mu_\cdot,\nu_\cdot). 
\end{equation}

Therefore, Theorem~\ref{thm:posterior} immediately yields an a posteriori error estimate in terms of $\mathcal{D}$.

\begin{thm}\label{thm}
Suppose that Assumption~\ref{assumption2} holds and that the initial distribution satisfies $\mu_0\in\mathcal{P}_2$. Let $\mu_\cdot^\ast$ be the distribution-valued function over $[0,T]$ corresponding to the solution of~\eqref{mean field equation}. If $\alpha>2C_0$, then, for any $\mu_\cdot\in \mathcal{P}_{2,\infty}([0,T])$, we have
$$
\frac{e^{-\alpha T/2}}{1+\kappa_\alpha}
\mathcal{D}(\mu_\cdot,\Phi(\mu_\cdot))
\leq
\mathcal{D}(\mu_\cdot,\mu_\cdot^*)
\leq
\frac{e^{\alpha T/2}}{1-\kappa_\alpha}
\mathcal{D}(\mu_\cdot,\Phi(\mu_\cdot)).
$$
\end{thm}
\begin{proof}
The result follows directly from Theorem~\ref{thm:posterior} and the equivalence \eqref{equivalence}
on the finite time interval $[0,T]$.
\end{proof}
Thus, the self-consistency error $\mathcal{D}(\mu_\cdot,\Phi(\mu_\cdot))$ is equivalent to the true error $\mathcal{D}(\mu_\cdot,\mu_\cdot^*)$ up to constants depending on $T$, $\alpha$, and the Lipschitz constant in Assumption~\ref{assumption2}.

\begin{remark}
    Although the computation of $\mathcal{D}(\mu_\cdot,\Phi(\mu_\cdot))$ can be challenging, especially for high-dimensional problems, the above result still provides valuable insight into the convergence behavior of the algorithm.
    For example, at selected time points, we can approximate $\Phi(\mu_t)$ by collecting the empirical distribution of particles simulated over multiple epochs and then compare this distribution with that of the current neural network solution $\mu_t$.
    If a significant discrepancy is observed between the two distributions, then the numerical solution is still far from the true solution; in this case, further training or adjustments to the network architecture, such as increasing its depth or width, may be necessary to enhance the model's expressive capacity.
\end{remark}

\section{Numerical experiments}\label{neumerical results}
In this section, we apply the DeepSPoC algorithm to a variety of equations to evaluate its performance. 
The numerical experiments cover mean-field SDEs with different spatial dimensions and diverse forms of $b$ and $\sigma$, together with their corresponding nonlinear Fokker–Planck equations.
The strong performance of DeepSPoC on these problems demonstrates its effectiveness, broad applicability, and computational advantages.

For the temporal KRnet, we employ 10 affine coupling blocks. 
Each block contains two fully connected hidden layers, each hidden layer consisting of 512 neurons.
In some examples, we also use fully connected neural networks to approximate the distribution. Specifically, we adopt a fully connected network with six hidden layers, each comprising 512 neurons.
All numerical experiments are implemented in PyTorch and accelerated using GPUs.

To evaluate the accuracy of the numerical solution, we adopt the relative $L^2$ error, estimated via Monte Carlo integration:
$$
\mathcal{E}_2:= \frac{\Vert f^{\text{ext}}-f^{\text{num}} \Vert_{L^2}}{\Vert f^{\text{ext}} \Vert_{L^2}}\approx\frac{\sqrt{\sum_{i=1}^{N_\mathcal{E}}\lvert f^{\text{ext}}(x_i)-f^{\text{num}}(x_i)}\rvert^2}{\sqrt{\sum_{i=1}^{N_\mathcal{E}}\lvert f^{\text{ext}}(x_i) \rvert^2}},
$$
where $\{ x_i \}_{i=1}^{N_e}$ are uniformly sampled from a predetermined region $\Omega_0$.
In the following experiments, we set $N_\mathcal{E}=10^5$.

We also evaluate the self-consistency error $\mathcal{D}(\rho_{\boldsymbol{\theta}},\Phi(\rho_{\boldsymbol{\theta}}))$, abbreviated as $\mathcal{D}$, using the following discrete approximation:
$$
\mathcal{D}(\rho_{\boldsymbol{\theta}},\Phi(\rho_{\boldsymbol{\theta}}))\approx \mathcal{D}_{\Delta t, N_{\mathcal{D}}}(\rho_{\boldsymbol{\theta}},\Phi(\rho_{\boldsymbol{\theta}})):=\left(  \Delta t \sum_{k=0}^M W^2_2\big(\hat{\rho}_{\boldsymbol{\theta},k},\widehat{\Phi(\rho_{\boldsymbol{\theta}})}_{k}\big) \right)^{1/2},
$$
where $\widehat{\rho}_{\boldsymbol{\theta},k}$ is the empirical distribution formed by $N_{\mathcal{D}}$ samples drawn from the learned density $\rho_{\boldsymbol{\theta}}(t_k,\cdot)$ at the current epoch, and $\widehat{\Phi(\rho_{\boldsymbol{\theta}})}_{k}$ is the empirical distribution of $N_{\mathcal{D}}$ independently simulated test particles at time $t_k$, obtained by substituting the learned density $\rho_{\boldsymbol{\theta}}$ into the coefficients of the mean-field SDE.
The Wasserstein distance $W_2$ is then computed between the two empirical point clouds using discrete optimal transport.
In the following experiments, we set $N_\mathcal{D}=10^4$.

\subsection{Porous medium equations}
Our first numerical experiment concerns the following porous medium equation (PME):
$$
\partial_t \rho =  \Delta\rho^m.
$$
This is a quasilinear parabolic equation that degenerates when $\rho=0$. 
For this equation, a classical solution does not always exist.
A well-known weak energy solution of the PME is the Barenblatt solution~(\ref{barenblatt}).
For $m>1$, the Barenblatt solution is defined by
\begin{equation}\label{barenblatt}
U(t,x)=\frac{1}{t^\alpha}\left\{ \left(C-\frac{m-1}{2m}\cdot\frac{\beta |x|^2}{t^{2\beta}}\right)^{+} \right\}^{\frac{1}{m-1}}
\end{equation}
where $x\in \mathbb{R}^d$, $\alpha=\frac{d}{d(m-1)+2}$, and $\beta=\frac{\alpha}{d}$. 
In our experiments, we set $m=3$ and $C=\frac{\sqrt{3}}{15}$. 

It can be observed that the Barenblatt solution (\ref{barenblatt}) is not normalized, since 
\[
c_0(d) := \int_{\mathbb{R}^d} U(t,x)\,dx \ne 1,
\]
where \(c_0(d)\) serves as the normalization constant depending on dimension \(d\).  
To obtain a probability density, we normalize it as \(\frac{1}{c_0}U(t,x)\), which satisfies the PME
\begin{equation}\label{normalize}
    \partial_t \rho = \nu \Delta \rho^m,
\end{equation}
with \(\nu = c_0^{\,m-1}\).
This PDE is associated with the density-dependent SDE
\begin{equation}\label{pme sde}
dX_t=\sqrt{2\nu}\rho^{\frac{m-1}{2}}(t,X_t)dB_t.
\end{equation}

We choose an initial time $t_0>0$ and take $\frac{1}{c_0}U(t_0,\cdot)$ as the initial distribution.  
DeepSPoC is then applied to solve the PME over a prescribed time interval, and the accuracy is evaluated by the relative $L^2$ error between the numerical and exact solutions at the terminal time.  
We present results for the porous medium equation in spatial dimensions $3,4,6$, and $8$.  
The cases in lower dimensions highlight a direct comparison between DeepSPoC and the classical PoC method, while the higher-dimensional cases demonstrate the applicability of DeepSPoC to challenging high-dimensional problems.

\subsubsection{3D and 4D results}
When solving the PME with three-dimensional spatial variables, we set the initial time $t_0=0.1$, the time span $T=0.2$, and the time discretization size $\Delta t=0.005$ (corresponding to 40 discrete time steps).  
For the four-dimensional PME, we set the initial time $t_0=1$, the time span $T=0.5$, and the time discretization size $\Delta t=0.025$ (corresponding to 20 discrete time steps).  
In both cases, the size of the particle batch simulated in each epoch is $1000$.
The initial learning rate of the optimizer is $0.001$, and it is multiplied by the contraction coefficient $\gamma = 0.5$ every $\Gamma = 500$ epochs.
We also compare the DeepSPoC solution with the PoC particle solution, obtained from the interacting particle system~\eqref{poc particle system for pme} with different particle numbers $N=10^6$, $3\times 10^6$, and $5\times 10^6$, under the same time discretization settings.  
Here, the mollification parameter $\epsilon$ in the Gaussian mollifier (see~\eqref{gaussian-kernel}) is set to $\epsilon=0.05$ for the 3D case and $\epsilon=0.1$ for the 4D case.
These values were selected after testing several candidates and were found to yield nearly the smallest relative $L^2$ errors among the tested choices. Although $\epsilon$ was not further optimized separately for each particle number $N$, additional tuning is not expected to substantially affect the comparison.

Table~\ref{tab:poc_deepspoc_comparison} compares the performance of PoC and DeepSPoC solutions in terms of error, solution storage, and computational time for different particle numbers.  
Since in each epoch DeepSPoC simulates $1000$ particles, the solution after $1000$ epochs corresponds to a DeepSPoC particle system~\eqref{deepspoc-particle-system-batchbybatch} with $10^6$ particles. 
We therefore compare it against the PoC particle solution with $N=10^6$, and similarly for other particle numbers.  
\begin{table}[htbp]
    \centering
    \caption{Comparison between PoC and DeepSPoC in 3D and 4D PME: 
            relative $L^2$ error, computational time (s) and solution storage (MB).}
    \label{tab:poc_deepspoc_comparison}
    \begin{tabular}{llccc}
        \toprule
        Dimension & Method & $N=10^6$/1000 epochs & $N=3\times10^6$/3000 epochs & $N=5\times10^6$/5000 epochs \\
        \midrule
        \multirow{2}{*}{3D} 
                & PoC       & \makecell{$5.08\times 10^{-2}$ \\$7.12\times 10^3$s / $456$MB}  & \makecell{$4.21\times10^{-2}$\\ $6.38\times 10^4$ s /$1372$MB } & \makecell{$3.96\times 10^{-2}$\\$1.81\times 10^5$s/$2288$MB}  \\
            & DeepSPoC  & \makecell{$8.01\times 10^{-2}$\\$2.41\times 10^3$s / $10.2$MB }  & \makecell{$4.11\times 10^{-2}$\\$7.24\times 10^3$s / $10.2$MB }  & \makecell{$3.67\times 10^{-2}$\\$1.17\times 10^4$s / $10.2$MB } \\
        \midrule
        \multirow{2}{*}{4D} 
            & PoC       & \makecell{$8.64\times 10^{-2}$\\$4.54\times 10^3$s/$304$MB}& \makecell{$7.73\times10^{-2}$\\$4.09\times 10^4$s/$914$MB} & \makecell{$7.12\times10^{-2}$\\$1.13\times 10^5$s/$1524$MB} \\
            & DeepSPoC  & \makecell{$1.09\times10^{-1}$\\$1.25\times 10^3$s/$10.2$MB} & \makecell{$5.42\times10^{-2}$\\$3.77\times 10^3$s/$10.2$MB}  & \makecell{$5.18\times 10^{-2}$\\$6.28\times 10^3$s/$10.2$MB} \\
        \bottomrule
    \end{tabular}
\end{table}
For solution storage, the PoC method requires storing the trajectories of all particles at all discrete time steps, whereas in DeepSPoC the solution is represented by the neural network itself. 
As the network architecture remains fixed, the storage cost of DeepSPoC is independent of both the number of particles and the spatial dimension.  

For computational time, we report the total wall-clock time required to complete the computation on the same hardware platform. Both PoC and DeepSPoC are accelerated on a single NVIDIA GeForce RTX 3090 GPU with 24 GB memory.
Although GPU execution allows parallel computation, the available memory on GPUs is limited. 
As a result, the PoC method must split all particles into batches and process them sequentially at each time step. 
Consequently, the runtime of PoC scales approximately as $O(N^2)$ with respect to the number of particles, matching the overall computational cost of PoC because all pairwise interactions among $N$ particles must be evaluated.
By contrast, for a fixed network architecture and batch size, the runtime of DeepSPoC increases approximately linearly with the number of epochs.
We remark that truncation techniques for density evaluation, such as the cell-list data structure~\cite[Appendix I]{cell-list}, can in principle reduce the interaction cost, but they introduce additional algorithmic components and may require spatial partitioning or grid-based data structures. 
Such techniques are problem-dependent and become increasingly costly to implement in higher dimensions. 
A systematic comparison with optimized fast particle solvers is beyond the scope of the present work.

Overall, these results demonstrate that DeepSPoC achieves a clear advantage in both storage cost and computational scalability, while maintaining accuracy comparable to or even higher than PoC, especially when handling large particle numbers and high-dimensional problems.

Fig.~\ref{3d4dfig} presents cross-sections along the first spatial variable of the PoC and DeepSPoC solutions with $N=5\times 10^6$ particles, corresponding to 5000 epochs. 
In the figure, the green dashed line represents the DeepSPoC solution, the red dashed line represents the PoC solution, the black solid line denotes the exact solution, and the gray solid line indicates the initial distribution. 
Compared with the mollified PoC particle solution, the DeepSPoC solution exhibits fewer oscillations and provides a sharper and more accurate representation near the interfaces.
The evolution of the relative $L_2$ error over training epochs for the 3D and 4D PME is shown in Fig.~\ref{fig:3D_D_and_err} and Fig.~\ref{fig:4Derr}, respectively. In both cases, the error decreases with fluctuations and gradually stabilizes as training proceeds.
In particular, Fig.~\ref{fig:3D_D_and_err} also shows the evolution of the self-consistency error $\mathcal{D}$ for the 3D PME. The relative $L_2$ error and $\mathcal D_{\Delta t,N_{\mathcal{D}}}$ exhibit similar decreasing trends, which supports the use of $\mathcal D_{\Delta t,N_{\mathcal{D}}}$ as a practical a posteriori error indicator. We note that the fluctuations and the remaining error of $\mathcal D_{\Delta t,N_{\mathcal{D}}}$ in the later stage of training mainly arise from the sampling error involved in the empirical estimation of the Wasserstein distance $W_2$. This sampling effect becomes more pronounced in higher spatial dimensions.

\begin{figure}[htbp]
    \centering
    \subfigure[3D PME solved by DeepSPoC]{
        \includegraphics[height=0.25\textwidth, width=0.32\textwidth]{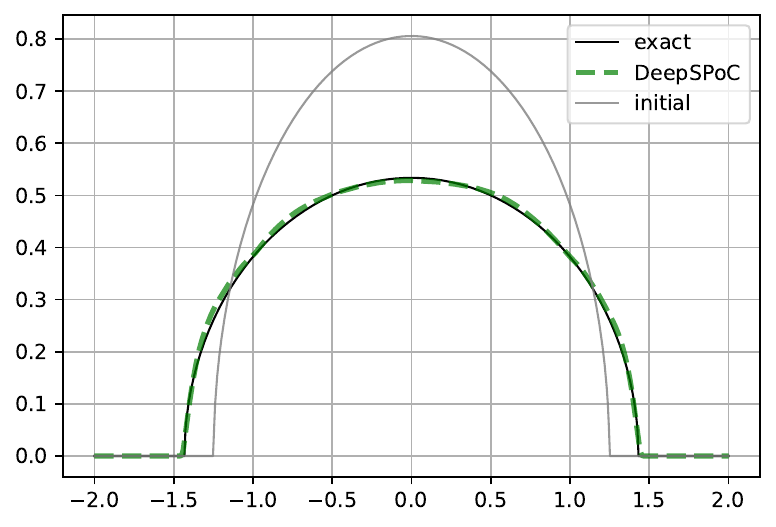}
        \label{fig:3dDeepSPoC}
    }
    \subfigure[3D PME solved by PoC]{
        \includegraphics[height=0.25\textwidth, width=0.32\textwidth]{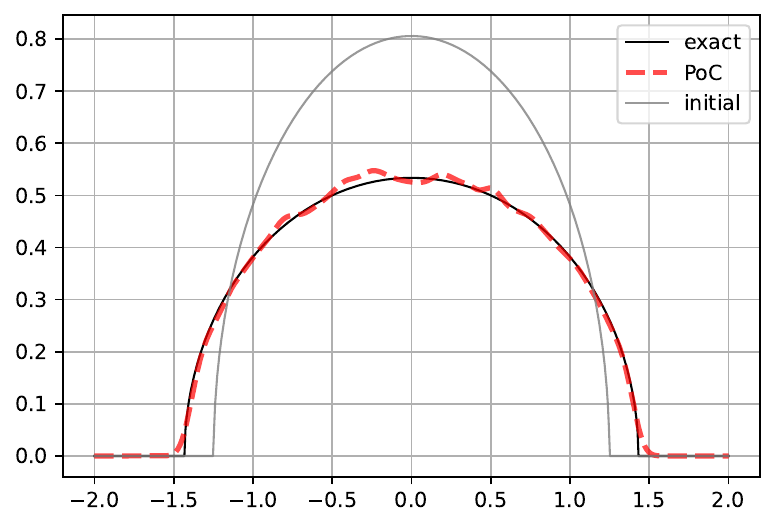}
        \label{fig:3dPoC}
    }
    \subfigure[$\mathcal{E}_2$ and $\mathcal{D}$ for the 3D PME]{
        \includegraphics[height=0.25\textwidth, width=0.30\textwidth]{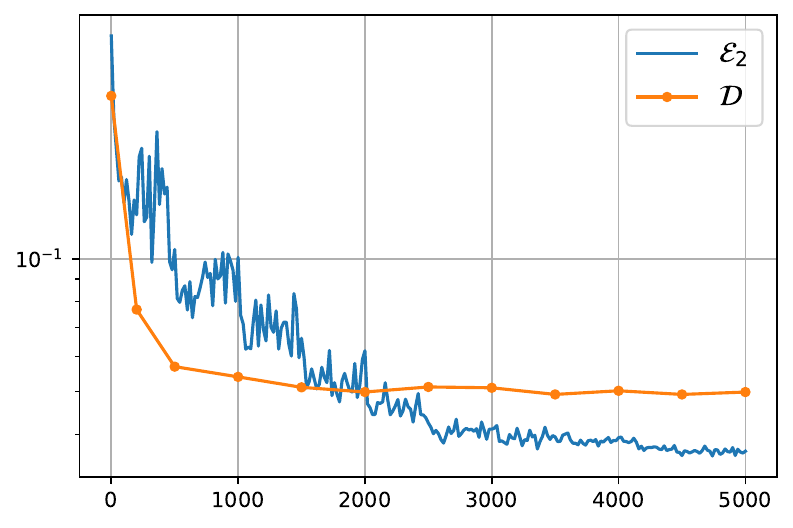}
        \label{fig:3D_D_and_err}
    }

    \subfigure[4D PME solved by DeepSPoC]{
        \includegraphics[height=0.25\textwidth, width=0.32\textwidth]{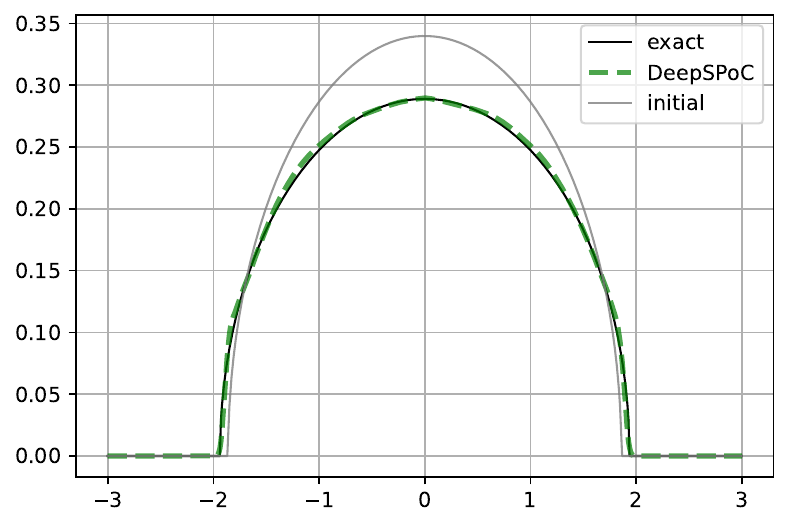}
        \label{fig:4dDeepSPoC}
    }
    \subfigure[4D PME solved by PoC]{
        \includegraphics[height=0.25\textwidth, width=0.32\textwidth]{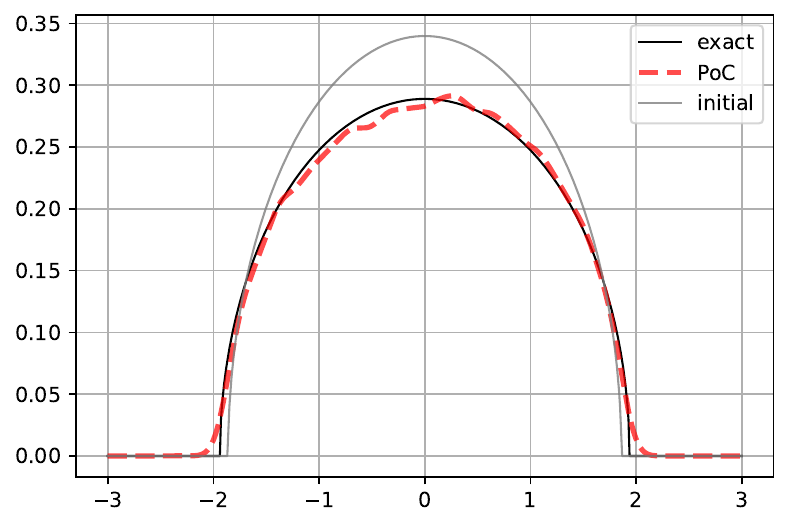}
        \label{fig:4dPoC}
    }
    \subfigure[$\mathcal{E}_2$ for the 4D PME]{
        \includegraphics[height=0.25\textwidth, width=0.30\textwidth]{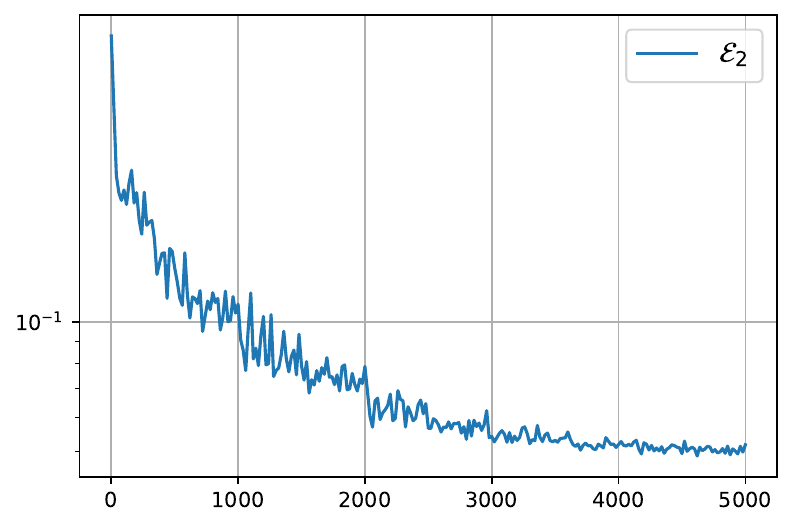}
        \label{fig:4Derr}
    }
    
    % \subfigure[Loss]{
    %     \includegraphics[height=0.25\textwidth, width=0.30\textwidth]{new_pictures/3D4Dloss.pdf}
    %     \label{fig:3D4dloss}
    % }
    \caption{Results of 3D and 4D porous medium equations.}
    \label{3d4dfig}
\end{figure}

In summary, the numerical experiments on 3D and 4D PME demonstrate that DeepSPoC not only reduces storage requirements and computational time compared with PoC, but also achieves comparable or even better accuracy. 

% \begin{figure}[htbp]
%     \centering
%     \subfigure[Comparison of exact and numerical solution at the cut of the first spatial variable to the 3D porous medium equation.]{
%         \includegraphics[scale=0.5]{new_pictures/3DPMEKRfig_with_ref.pdf}
%         \label{fig:3dpme}
%     }
%     \subfigure[The change of relative $L^2$ error with epochs.]{
%         \includegraphics[scale=0.5]{new_pictures/3dpmeKRerr.pdf}
%         \label{fig:3dpmeerr}
%     }
%     \caption{Result of 3D porous medium equation.}
% \end{figure}

\subsubsection{6D and 8D results}

For high-dimensional problems, we further apply the DeepSPoC algorithm to porous medium equations with 6- and 8-dimensional spatial variables.

For the 6D case, we set the initial time $t_0=1$, the time span $T=0.5$, and the time discretization size $\Delta t=0.025$.  
The size of the particle batch simulated in each epoch is $4000$.  
The initial learning rate of the optimizer is set to $0.001$ and is multiplied by the contraction coefficient $\gamma=0.5$ every $\Gamma=500$ epochs.  

For the 8D case, we set the initial time $t_0=1$, the time span $T=0.5$, and the time discretization size $\Delta t=0.025$.  
The size of the particle batch simulated in each epoch is $6000$.  
The initial learning rate of the optimizer is set to $0.001$ and is multiplied by the contraction coefficient $\gamma=0.7$ for every $\Gamma=500$ epochs.  
 
Fig.~\ref{6d8dfig} presents the results of the DeepSPoC method for 6D and 8D PME. 
For the 6D case, the plotted solution corresponds to $5000$ training epochs, whereas for the 8D case, the results are obtained after $8000$ epochs.
The cut-view plots along the first spatial variable illustrate that the DeepSPoC solutions (green dashed lines) closely match the exact solutions (black solid lines), with the initial distribution shown in gray.  
In addition, the evolution of the relative $L^2$ error with respect to epochs is also displayed, showing a decreasing and stable convergence trend as the number of epochs increases.  

These results demonstrate that DeepSPoC maintains accuracy in high-dimensional settings (6D and 8D), highlighting its ability to solve nonlinear Fokker–Planck equations in higher dimensions.

% \begin{figure}[htbp]
%     \centering
%     \subfigure[6D DeepSPoC]{
%         \includegraphics[scale=0.5]{new_pictures/6DPMEKRfig_DeepSPoC_test.pdf}
%         \label{fig:6dpme}
%     }
%     \subfigure[8D DeepSPoC]{
%         \includegraphics[scale=0.5]{new_pictures/8DPMEKRfig_DeepSPoC_test.pdf}
%         \label{fig:8dpme}
%     }
%     \subfigure[$\mathcal{E}_2$]{
%         \includegraphics[scale=0.5]{new_pictures/6D8Derr_log.pdf}
%         \label{fig:err_compare}
%     }
%     \caption{Results of 6D and 8D porous medium equations.}
%     \label{6d8dfig}
% \end{figure}

\begin{figure}[htbp]
    \centering
    \subfigure[6D PME solved by DeepSPoC]{
        \includegraphics[height=0.25\textwidth, width=0.32\textwidth]{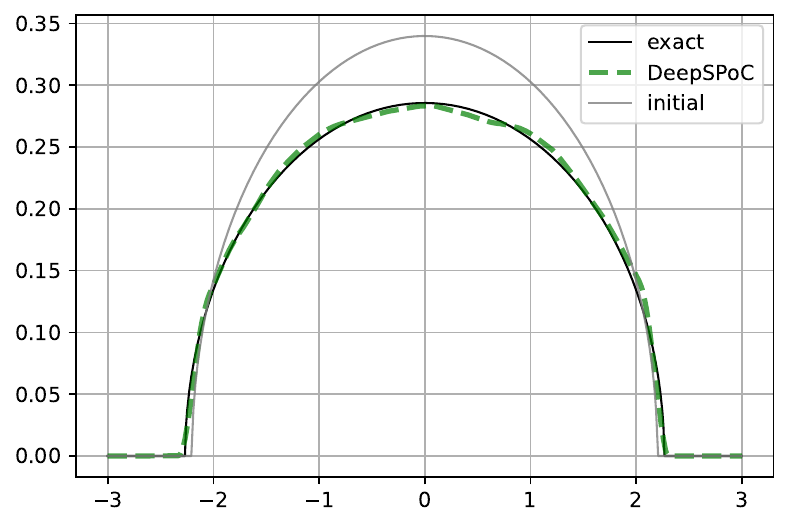}
        \label{fig:6dpme}
    }
    \hfill
    \subfigure[8D PME solved by DeepSPoC]{
        \includegraphics[height=0.25\textwidth, width=0.32\textwidth]{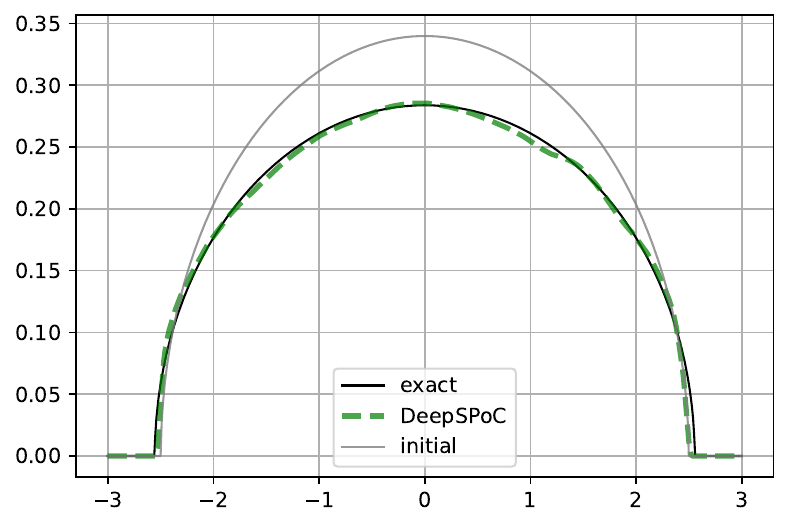}
        \label{fig:8dpme}
    }
    \hfill
    \subfigure[$\mathcal{E}_2$]{
        \includegraphics[height=0.25\textwidth, width=0.30\textwidth]{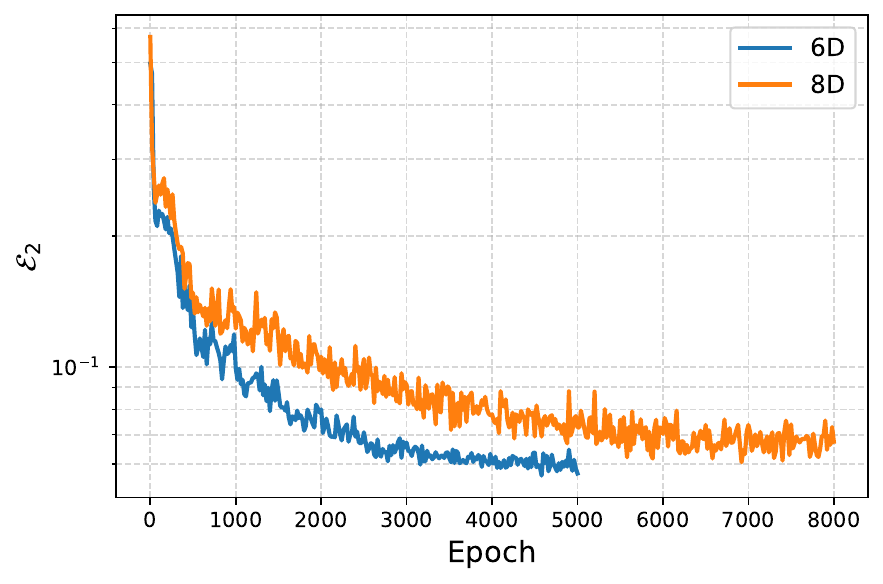}
        \label{fig:err_compare}
    }
    \caption{Results of 6D and 8D porous medium equations.}
    \label{6d8dfig}
\end{figure}

\subsection{Keller--Segel equations}
In this subsection, we apply both the KRnet model and the fully connected network (see Remark~\ref{fully connected network method}) to the Keller--Segel equation
\[
\partial_t \mu = \nabla \cdot \left( (\nabla W \ast \mu)\mu \right) + \Delta \mu, 
\quad (x,t)\in \mathbb{R}^d\times (0,+\infty),
\]
where
\begin{align}\label{W case}
    W(x)=
    \begin{cases}
        \tfrac{1}{2\pi}\ln(|x|), & d=2, \\[6pt]
        -\tfrac{C_d}{|x|^{d-2}}, & d\ge 3,
    \end{cases}
\end{align}
with $C_d=\tfrac{1}{d(d-2)\alpha_d}$ and $\alpha_d$ denoting the volume of the unit ball in $\mathbb{R}^d$.  
This PDE is a nonlocal semilinear equation because of the convolution term that involves $\mu$, and can be associated with the mean-field SDE
\[
dX_t = -\left( \nabla W \ast \mu_t \right)(X_t)\,dt + \sqrt{2}\,dB_t,
\]
where the convolution term describes particle interactions. The interaction kernel $\nabla W$ is attractive and singular.  
For the 2D case, the propagation of chaos (PoC) result for this Keller--Segel chemotaxis model has been established in~\cite{pocforKS}.  

When simulating the SDEs in the DeepSPoC algorithm using the Euler scheme~\eqref{euler-scheme}, the convolution must be evaluated via Monte Carlo integration at each discrete time step $t_m$.  
Specifically, when fully connected networks are employed, we generate independent samples $\{x_i\}_{i=1}^{N_g}$ from the approximate density $\rho_{FC,\boldsymbol{\theta}}(t_m,\cdot)$ using the accept--reject sampling method.  
In contrast, when using KRnet (Algorithm~\ref{Deep SPoC Algorithm with path}), samples $\{x_i\}_{i=1}^{N_g}$ can be drawn directly from the density $\rho_{NF,\boldsymbol{\theta}}(t_m,\cdot)$.  
In both cases, the generated particles follow the distribution $\mu_{t_m}^{\boldsymbol{\theta}}$, and the convolution is approximated as
\begin{equation}\label{monte carlo for convolution}
    \nabla W \ast \mu_{t_m}^{\boldsymbol{\theta}}(x) 
    \;\approx\; \frac{1}{N_g}\sum_{i=1}^{N_g}\nabla W(x-x_i).
\end{equation}

For the Keller--Segel equation, no explicit closed-form solution is available, so direct comparison with a ground-truth solution is not possible.  
Nevertheless, a well-known property of the 2D Keller--Segel equation is that the second moment of the solution evolves linearly in time~\cite{K-S2MOMENT}, with the slope given explicitly by $4\left(1-\tfrac{1}{8\pi}\right)$.  
We exploit this property to assess the accuracy of DeepSPoC in solving the 2D Keller--Segel equation.  
In particular, during each epoch, we record the second moment of the distribution at every discrete time step, estimated via Monte Carlo integration, and then compute its slope using the least-squares method. We validate the effectiveness of our algorithm by checking whether the evolution of the second moment aligns with a straight line with the theoretical slope.

The first initial distribution considered in our experiments is a two-dimensional Gaussian distribution with density
\[
\rho_0(x_1,x_2) = \frac{1}{0.36\pi}\exp\!\left(-\frac{x_1^2+x_2^2}{0.36}\right).
\] 
As a more complicated case, we also use a mixture of two weighted Gaussian distributions:
\[
\rho_0(x_1,x_2)=\tfrac{1}{3}\cdot\frac{1}{0.36\pi}\exp\!\left(-\tfrac{(x_1+1.5)^2+x_2^2}{0.36}\right)
+\tfrac{2}{3}\cdot\frac{1}{0.36\pi}\exp\!\left(-\tfrac{(x_1-1)^2+x_2^2}{0.36}\right).
\]

For the numerical setup, we set $N_g=500$ in~(\ref{monte carlo for convolution}). 
The total time span is $T=0.2$ with a time step size $\Delta t=0.01$. 
Each epoch samples $K=2000$ particles. 
The optimizer starts with an initial learning rate of $0.001$, which is reduced by a contraction factor $\gamma=0.7$ every $\Gamma=500$ epochs.  
For the fully connected network version of DeepSPoC, the number of training points $K_0$ (see Remark~\ref{fully connected network method}) is set to $2000$, uniformly sampled from the region $[-4,4]^2$, and the mollifier parameter is chosen as $\epsilon=0.02$.  

The evolution of the solution after $8000$ epochs at different times is shown in Figs.~\ref{fig:2dKSfig}, \ref{fig:2dKSfig_mix}, \ref{fig:2dKSfigKR}, and~\ref{fig:2dKSfig_mixKR}, where we again plot cross-sections with respect to the first spatial variable. The second-moment slopes of the solutions computed with different neural networks and initial distributions over epochs are shown in Figs.~\ref{fig:2dKSslope}, \ref{fig:2dKSslope_mix}, \ref{fig:2dKSslope_kr}, and~\ref{fig:2dKSslope_mix_kr}, respectively. 
We observe that both network architectures solve the problem with reasonable accuracy, while the KRnet-based approach achieves faster convergence and exhibits fewer oscillations.

\begin{figure}[htbp]
    \centering
    \subfigure[Cross-sections of the numerical solution along the first spatial variable at different times.]{
        \includegraphics[scale=0.45]{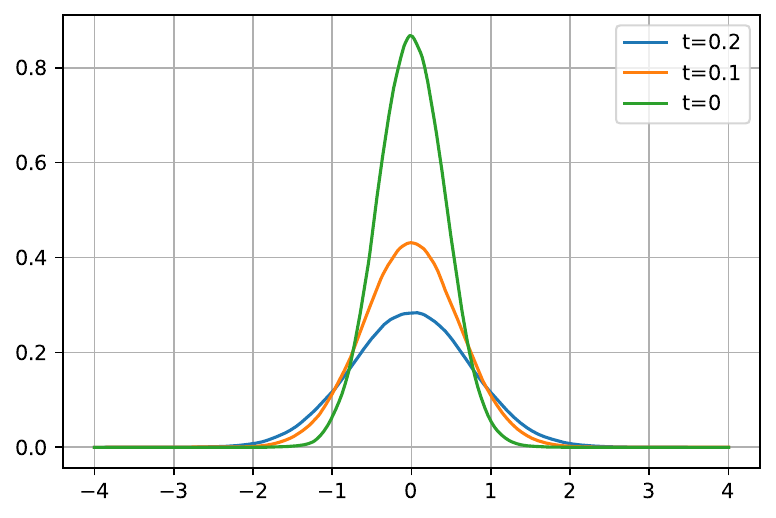}
        \label{fig:2dKSfig}
    }
    \quad
    \subfigure[Evolution of the second-moment slope of the numerical solution over epochs, compared with the exact slope.]{
        \includegraphics[scale=0.45]{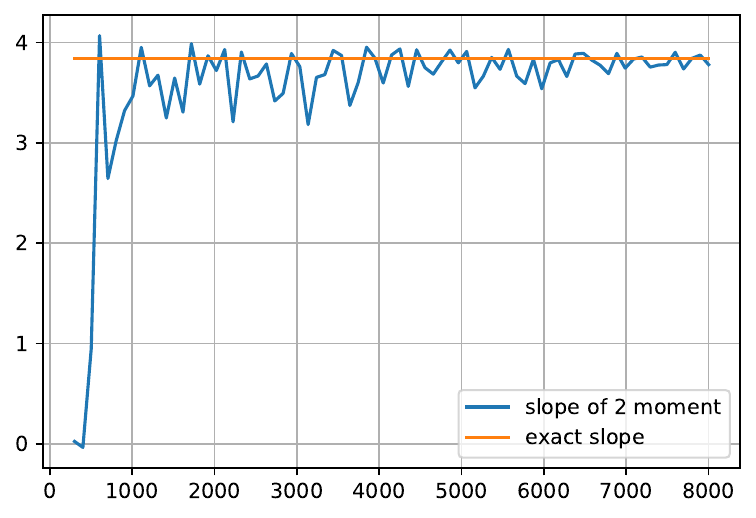}
        \label{fig:2dKSslope}
    }
    \subfigure[Cross-sections of the numerical solution along the first spatial variable at different times.]{
        \includegraphics[scale=0.45]{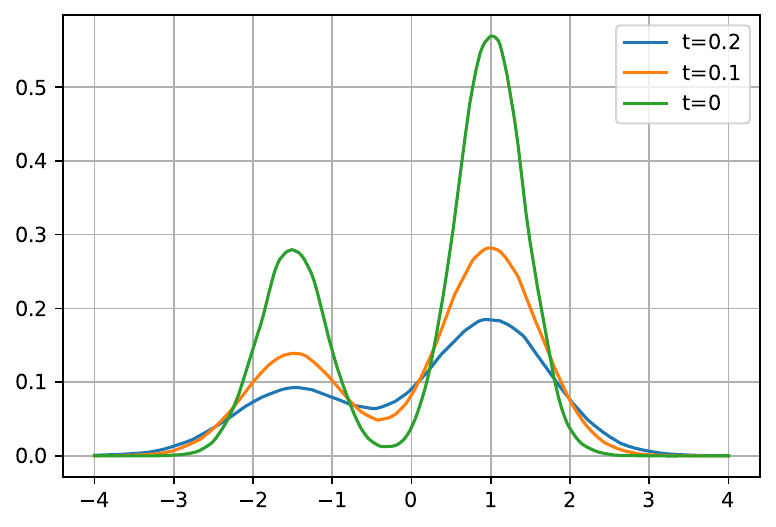}
        \label{fig:2dKSfig_mix}
    }
    \quad
    \subfigure[Evolution of the second-moment slope of the numerical solution over epochs, compared with the exact slope.]{
        \includegraphics[scale=0.45]{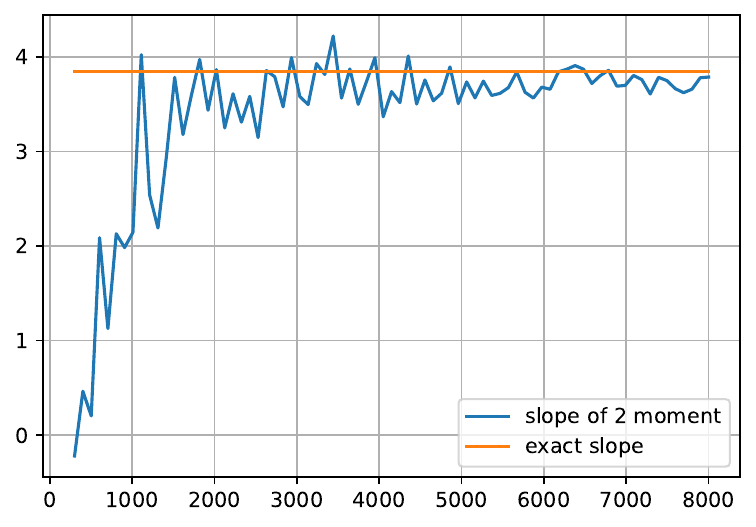}
        \label{fig:2dKSslope_mix}
    }
    
    \caption{Results for the 2D Keller--Segel equation with different initial distributions, computed by DeepSPoC with fully connected neural networks.}
\end{figure}

\begin{figure}[htbp]
    \centering
    \subfigure[Cross-sections of the numerical solution along the first spatial variable at different times.]{
        \includegraphics[scale=0.45]{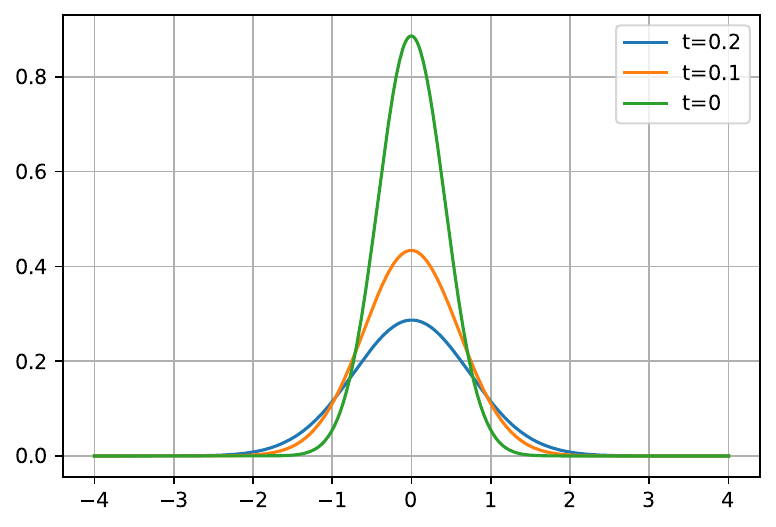}
        \label{fig:2dKSfigKR}
    }
    \quad
    \subfigure[Evolution of the second-moment slope of the numerical solution over epochs.]{
        \includegraphics[scale=0.45]{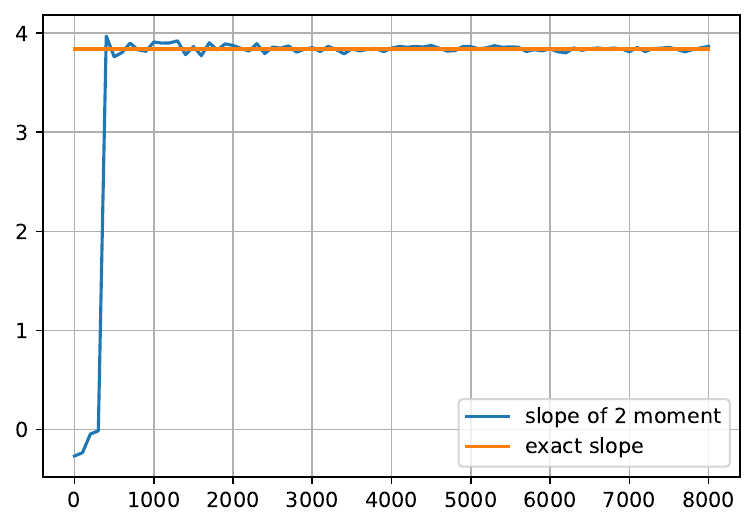}
        \label{fig:2dKSslope_kr}
    }
    \subfigure[Cross-sections of the numerical solution along the first spatial variable at different times.]{
        \includegraphics[scale=0.45]{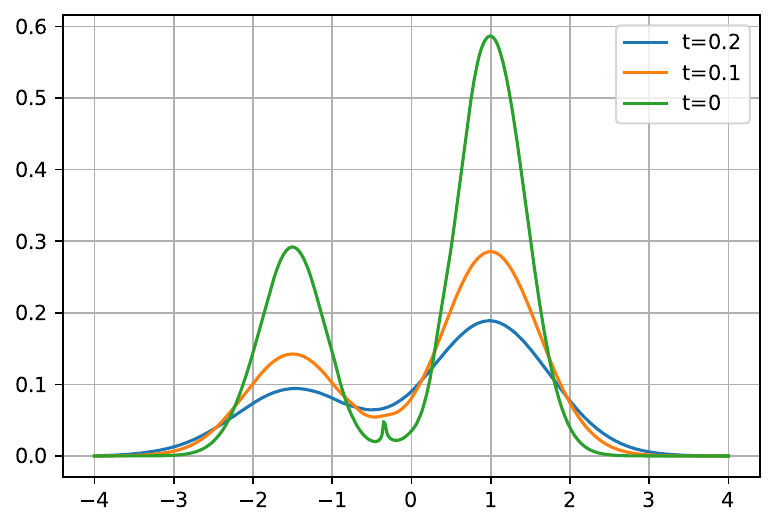}
        \label{fig:2dKSfig_mixKR}
    }
    \quad
    \subfigure[Evolution of the second-moment slope of the numerical solution over epochs.]{
        \includegraphics[scale=0.45]{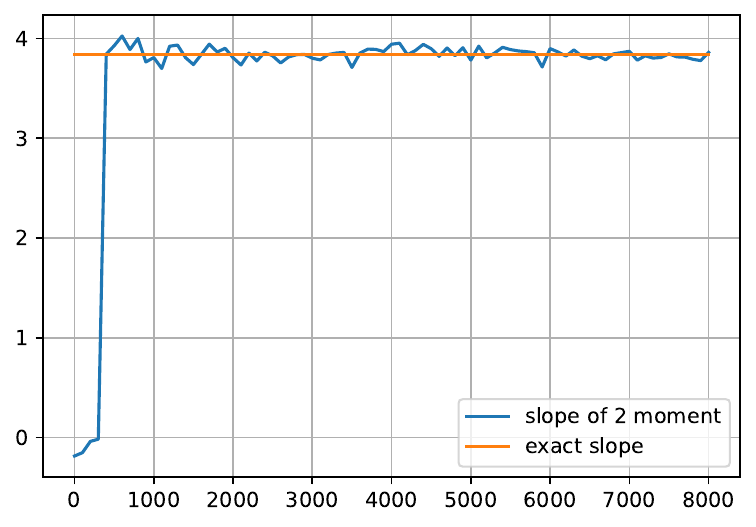}
        \label{fig:2dKSslope_mix_kr}
    }
    
    \caption{Results for the 2D Keller--Segel equation with different initial distributions, computed by DeepSPoC with KRnet.}
\end{figure}

\subsection{Curie–Weiss mean-field equations}
This example is known as the Curie-Weiss mean-field equation. 
It is an expectation-dependent mean-field SDE of the following form: 
$$
dX_t=\left( -\beta(|X_t|^2X_t-X_t)+\beta K\mathbb{E}X_t \right)dt+dB_t.
$$
If $\beta$ and $K$ satisfy certain conditions, the equation admits an explicit invariant measure with the following density (cf.~\cite{cao2023empirical}):
$$
p^{\ast}(x)=\frac{1}{C}\text{exp}\left(-2\beta\Big(\frac{|x|^4}{4}-\frac{|x|^2}{2}\Big)\right),
$$
where $C=\int_{\mathbb{R}}\text{exp}\left(-2\beta(\frac{x^4}{4}-\frac{x^2}{2})\right)dx$.
In this example, we study the long-time asymptotic behavior of the DeepSPoC solution to the above equation in the one- and two-dimensional cases, using a PoC solution as the reference.
The PoC solution is the mollified empirical measure generated by 5 million particles, while the DeepSPoC solution is generated by 5000 training epochs, each using 1000 sampled particles. 
Thus, the two solutions can be regarded as being generated from the same total number of particles.
% Specifically, the reference PoC solution which we show and compare with DeepSPoC solution at time $t$ is the mollified empirical measure.
% $
% \frac{1}{5\times10^6}\sum_{i=1}^{5\times10^6}\delta_{X^i_{t}}\ast f_\epsilon.
% $

For the 1D case, we use a fully connected neural network and choose $\mathcal{N}(1,1)$ as the initial distribution, which is a standard Gaussian distribution shifted one unit to the right. 
We set $\beta=1$ and $K=-0.1$.
The time discretization size is 0.01 for both the PoC and DeepSPoC methods.
Since we cannot sample directly from the fully connected neural network, we use accept--reject sampling in DeepSPoC to estimate the expectation $\mathbb{E}X_t$ at each discrete time point using 100 sampled particles. 
The number of training points $K_0$ in Remark~\ref{fully connected network method} is 1000 for each epoch, and these points are sampled uniformly from the truncated region $\Omega_0=[-3,3]$. 
The initial learning rate is again $0.001$ and is reduced by half every 500 epochs. 
The parameter $\epsilon$ for the Gaussian mollifier is $0.01$. 
The comparison of the two solutions at $t=2,4,6,8,10$ is shown in the first five subplots of Fig.~\ref{CWfig}, while the last subplot shows the invariant measure of this equation. 
The figure shows that the DeepSPoC solution (blue line) is almost identical to the PoC solution (red dashed line) and gradually converges to the invariant measure.
\begin{figure}[htbp]
            \centering
        \includegraphics[height=60mm,width=110mm]{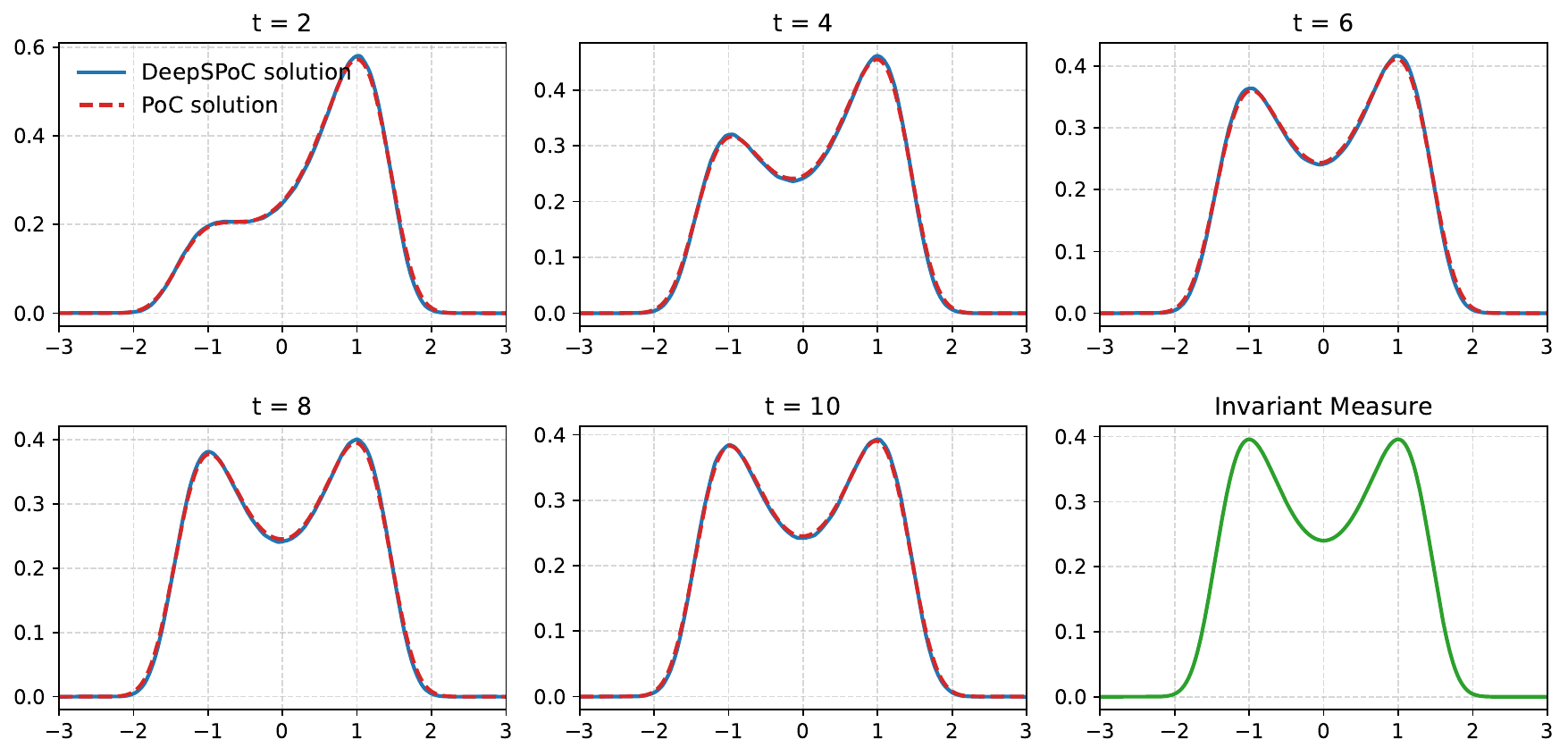}
        \caption{Numerical solutions of the 1D CW equation at $t=2,4,6,8,10$ and its invariant measure.}
        \label{CWfig}
\end{figure}

For the 2D case, we use temporal KRnet to solve the equation, and the initial points are sampled from $\mathcal{N}((-0.3,-0.3),0.36I_2)$.
We set $\beta=1$ and $K=0.1$, and the time discretization size is 0.025 for both the PoC and DeepSPoC methods.
We sample $1000$ particles from KRnet at each discrete time point to estimate $\mathbb{E}X_t$.
Similarly, the initial learning rate is $0.001$ and is reduced by half every 500 epochs. 
Fig.~\ref{2DCWfig} shows the comparison between the reference PoC solution and the DeepSPoC solution of the 2D Curie--Weiss equation at $t=0,2,4$, and $6$. The results show that DeepSPoC accurately captures the time evolution of the density and agrees well with the reference PoC solution at different time levels. We also present the invariant measure of the equation, which illustrates the long-time equilibrium profile toward which the DeepSPoC solution converges. The last subplot of Fig.~\ref{2DCWfig} shows the evolution of the self-consistency error $\mathcal{D}$ and the relative $L^2$ error between the final-time density and the invariant measure over training epochs. These two quantities exhibit similar decreasing trends, supporting the use of $\mathcal{D}{\Delta t,N{\mathcal{D}}}$ as a practical a posteriori error indicator.

\begin{figure}[htbp]
  \centering
  % 例如：横向只占正文宽度的 0.85 倍，纵向占 0.7 倍
  \includegraphics[width=0.93\textwidth,height=0.8\textheight,keepaspectratio]{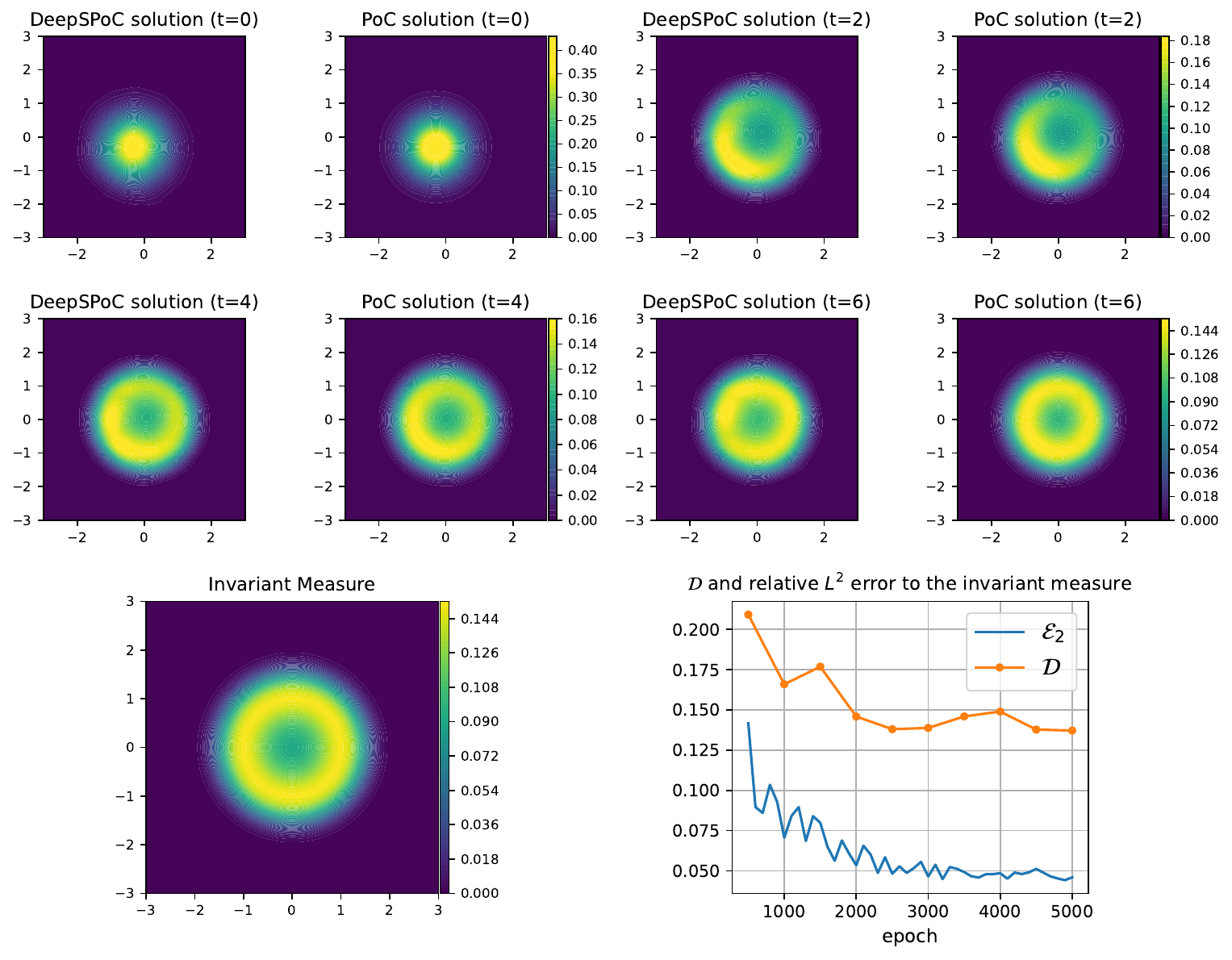}
  \caption{Numerical results for the 2D CW equation.}
  \label{2DCWfig}
\end{figure}

\subsection{Fractional porous medium equations}
The last example concerns the following fractional porous medium equation (FPME):
$$
\partial_t\rho = -(-\Delta)^{\frac{\alpha}{2}}(|\rho|^{m-1}\rho)
$$
where $m>1$, $\alpha\in (0,2)$, and $(-\Delta)^{\frac{\alpha}{2}}$ is the fractional Laplacian defined by 
$$
(-\Delta)^{\frac{\alpha}{2}}f(x)= C_{d,\alpha}\text{P.V.}\int_{\mathbb{R}^d}\frac{f(x)-f(y)}{|x-y|^{d+\alpha}}dy
$$
where the normalization constant is $C_{d,\alpha}=2^{\alpha-1}\alpha\Gamma(\frac{d+\alpha}{2})/\pi^{N/2}\Gamma(1-\frac{\alpha}{2})$. 
This is a degenerate and nonlocal parabolic equation that arises naturally as the limiting equation of particle systems with jumps or long-range interactions. 
It is proved in~\cite{FPMEtheory} that the weak solution of the above FPME can be represented as the distributional density of the solution to the following distribution-dependent SDE driven by the $\alpha$-stable process $L^{\alpha}_t$:
$$
dX_t=\rho(t,X_{t-})^{\frac{m-1}{\alpha}}dL^{\alpha}_t.
$$
In our experiments, we set $\alpha=1$ and $m=2$ and solve the equation in the one- and three-dimensional cases. 

In the 1D case, we compare the DeepSPoC solution obtained using a fully connected neural network with a reference solution computed by the finite difference method introduced in~\cite{del2014finite}. 
% It should be pointed out that we can also apply DeepSPoC to higher-dimensional FPMEs. 
% However, from~\cite{del2014finite} we can only compute reference solutions of 1D FPMEs according to the method in this paper we need to expand the space dimension from $d$ to $d+1$ then apply finite difference method to the expanded equation, which makes the computational costs for solving FPMEs in two or more dimensions become too large. 
% Therefore, we only compare the DeepSPoC solution and the reference solution in the one-dimensional case of FPMEs.
The discrete time step for DeepSPoC is $0.01$ and the number of sampled particles in each epoch is $2000$. 
The number of training points $K_0$ (see Remark~\ref{fully connected network method}) is $2000$, and these points are uniformly sampled from the region $\Omega_0=[-3,3]$. 
The initial learning rate is $0.001$ and is reduced by half every 500 epochs. 
We choose a normalized Barenblatt solution as the initial distribution so that the different behaviors of the FPME and PME solutions can be observed more clearly. 
The results are shown in Figs.~\ref{fig:fpme0.3} and~\ref{fig:fpme0.5}, which compare the DeepSPoC solution after $5000$ epochs (green line) with the reference solution (red dashed line) at two different time points; the reference solution is computed with spatial step size $0.005$ in both the original and extended dimensions and time step size $0.005$. 
We observe that the DeepSPoC solution is close to the reference solution and exhibits infinite propagation speed and a long-tail profile, in contrast to the PME case.  

\begin{figure}[!tb]
    \centering
    \subfigure[1D FPME at $t=0.3$.]{
        \includegraphics[height=0.27\textwidth, width=0.31\textwidth]{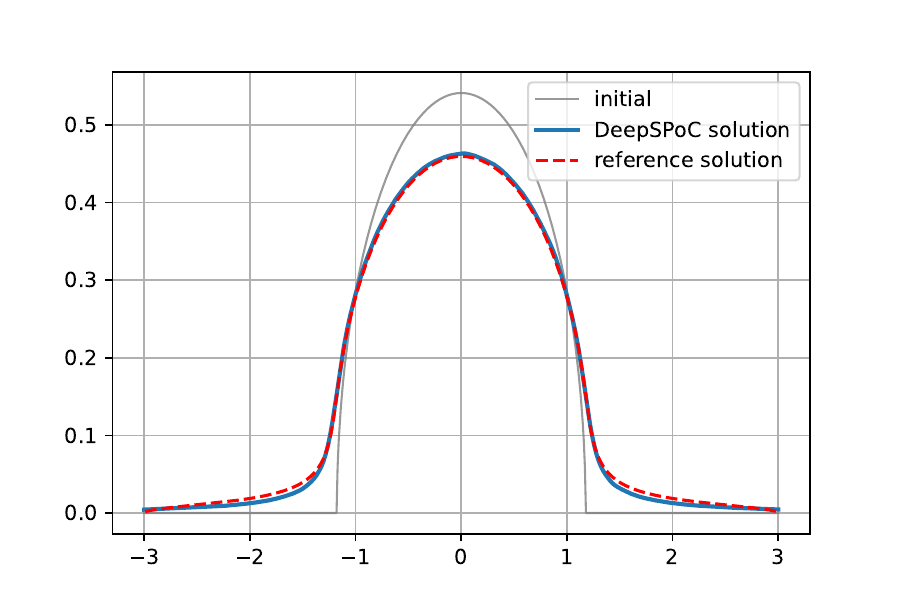}
        \label{fig:fpme0.3}
    }
    \subfigure[1D FPME at $t=0.5$.]{
        \includegraphics[height=0.27\textwidth, width=0.31\textwidth]{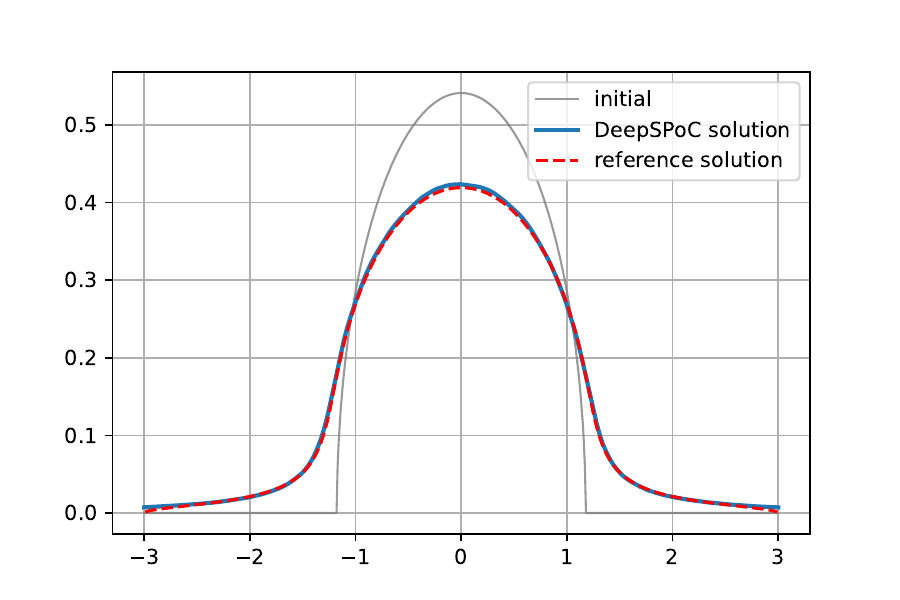}
        \label{fig:fpme0.5}
    }
    % \caption{Results of 1D fractional porous medium equation for different time.}
    \subfigure[3D FPME at $t=0.3$ (cross-section along the first spatial variable).]{
        \includegraphics[height=0.24\textwidth, width=0.27\textwidth]{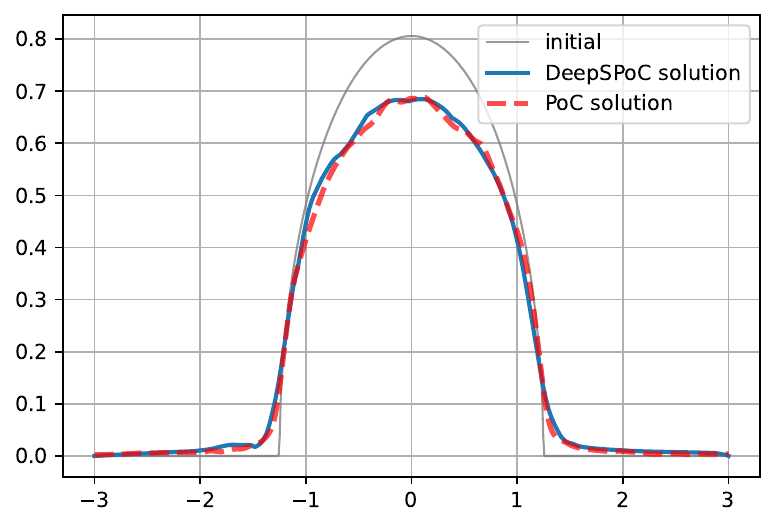}
        \label{3DFPMEKR}
    }
    \caption{Results of fractional porous medium equations.}
    \label{fig:fpme}
\end{figure}

When solving the three-dimensional case, the finite difference method becomes infeasible due to the high dimensionality. 
Therefore, we compare the DeepSPoC solution obtained using KRnet with the particle-based PoC solution. 
The Barenblatt solution is still used as the initial condition. 
The total time span is $T=0.3$, with a time discretization step of $\Delta t=0.01$, and each epoch simulates a particle batch of size 4000.
The optimizer is initialized with a learning rate of 
0.001, which is multiplied by a contraction coefficient $\gamma=0.7$ every $\Gamma=500$ epochs. 
The training process runs for a total of 7000 epochs.
For comparison, the PoC solution is computed using 5,000,000 particles, a mollifier parameter of $\epsilon=0.05$, and the same time span and time discretization size as those used for DeepSPoC.
The result is shown in Fig.~\ref{3DFPMEKR}.
In the figure, we plot cross-sections of the DeepSPoC and PoC solutions with respect to the first spatial variable at the terminal time. 
In the three-dimensional case, the DeepSPoC solution is close to the particle-based reference solution, indicating that our method remains effective for solving the three-dimensional FPME.
% \begin{figure}[!tb]
%     \centering
%     \subfigure[Comparison of the numerical solution's cut of the first spacial variable.]{
%         \includegraphics[scale=0.48]{new_pictures/3DFPMEKRfig.pdf}
%         \label{3DFPMEKR}
%     }
%     \caption{Results of 3D fractional porous medium equation.}
%     \label{fig:3DFPMEKR}
% \end{figure}

\section{Conclusion}
We proposed an efficient algorithm for nonlinear Fokker–Planck-type mean-field equations, built on SPoC theory and implemented using the KRnet model. 
By replacing particle–particle interactions with particle–network interactions, the method greatly improves scalability and reduces both the interaction cost and the storage burden in large-scale particle simulations. 
Numerical results for multiple types of mean-field models demonstrate the effectiveness of the proposed approach and show clear advantages over the traditional PoC method, particularly for high-dimensional and complex problems. 
A posteriori error estimate is also provided to estimate the distance between the numerical solution and the true solution.
Future work will focus on further theoretical analysis of the SPoC iterative formulation, improved optimization strategies based on spatial–temporal adaptivity in particle trajectory simulation, and extensions of the method to broader kinetic models such as the Vlasov–Poisson–Fokker–Planck and Boltzmann equations.

\section*{Acknowledgements}
The first author was supported by the National Science and Technology Major Project (No. 2022ZD0116401), the National Natural Science Foundation of China (No. 12222103), and LMNS at Fudan University.
The second author was supported by the National Science and Technology Major Project (No.2022ZD0116401).
The third author was supported by the NSF of China (under grants 12288201 and 12461160275) and the high-level talent research start-up project funding of Henan Academy of Sciences (No. 232019024). The fourth author was supported by the Youth Program of the National Natural Science Foundation of China (NSFC) (No. 12401476) and the Pujiang Project under the Shanghai Baiyulan Talent Program (No. 25PJA113).

\appendix
\section{Proof of Proposition \ref{contraction proposition}}
Using the synchronous coupling method, for $\mu_{\cdot},\nu_{\cdot}\in \mathcal{P}_{2,\infty}([0,T])$, we introduce the following two Markovian SDEs:
\begin{align*}
    \begin{aligned}
    dX_t & = b(t,X_t,\mu_t)dt+\sigma(t,X_t,\mu_t)dB_t,\\
    dY_t & = b(t,Y_t, \nu_t)dt+\sigma(t,Y_t,\nu_t)dB_t,
   \end{aligned}
\end{align*}
where $X_0=Y_0\sim \mu_0$. Then we have
\begin{align*}
    \begin{aligned}
    \mathbb{E}|X_t-Y_t|^2 & \le 2\mathbb{E}\left| \int_0^tb(s,X_s,\mu_s)-b(s,Y_s,\nu_s)ds \right|^2 + 2\mathbb{E}\left| \int_0^t \sigma(s,X_s,\mu_s)-\sigma(s,Y_s,\nu_s)dW_s \right|^2\\
    & \le 2t\mathbb{E}\int_0^t|b(s,X_s,\mu_s)-b(s,Y_s,\nu_s)|^2ds+2\mathbb{E}\int_0^t\Vert\sigma(s,X_s,\mu_s)-\sigma(s,Y_s,\nu_s)\Vert^2ds\\
    & \le 2(T+1)C^2 \left(\mathbb{E}\int_0^t|X_s-Y_s|^2ds + \mathbb{E}\int_0^tW_2^2(\mu_s,\nu_s)ds\right).
    \end{aligned}
\end{align*}
Multiplying both sides by $e^{-\alpha t}$ and integrating from $0$ to $T$, we obtain
\begin{align*}
    \begin{aligned}
    & \int_0^T e^{-\alpha t}\mathbb{E}|X_t-Y_t|^2dt\\
    \le & 2(T+1)C^2\int_0^T e^{-\alpha t}\left( \mathbb{E}\int_0^t |X_s-Y_s|^2ds \right)dt+2(T+1)C^2\int_0^T e^{-\alpha t}\left(\int_0^t W_2^2(\mu_s,\nu_s)ds\right)dt,
    \end{aligned}
\end{align*}
Exchanging the order of integration gives
\begin{align*}
    \begin{aligned}
     \int_0^T e^{-\alpha t}\left(\mathbb{E}\int_0^t|X_t-Y_t|^2ds\right)dt 
    = & \mathbb{E}\int_0^T\int_s^T e^{-\alpha(t-s)}e^{-\alpha s}|X_s-Y_s|^2dtds \\
    \le & \frac{1}{\alpha}\int_0^T e^{-\alpha s}\mathbb{E}|X_s-Y_s|^2ds.
    \end{aligned}
\end{align*}
Similarly, we obtain
$$
\int_0^T e^{-\alpha t}\left( \int_0^t W_2^2(\mu_s,\nu_s)ds \right)dt\le \frac{1}{\alpha}\int_0^T e^{-\alpha s}W_2^2(\mu_s,\nu_s)ds.
$$
Then we have
$$
\int_0^T e^{-\alpha t}\mathbb{E}|X_t-Y_t|^2dt\le \frac{2(T+1)C^2}{\alpha}\int_0^T e^{-\alpha t}\mathbb{E}|X_t-Y_t|^2dt+ \frac{2(T+1)C^2}{\alpha}\int_0^T e^{-\alpha t}W_2^2(\mu_t,\nu_t)dt,
$$
Using the property $W_2^2(\text{Law}(X_t),\text{Law}(Y_t)) \le \mathbb{E}|Y_t-X_t|^2$ and recalling the definition of $\Phi$, we obtain
$$
\left( 1-\frac{2(T+1)C^2}{\alpha} \right)\int_0^T e^{-\alpha t}W_2^2(\Phi(\mu_{\cdot})_t,\Phi(\nu_{\cdot})_t)dt\le \frac{2(T+1)C^2}{\alpha}\int_0^T e^{-\alpha t}W_2^2(\mu_t,\nu_t)dt.
$$
The desired conclusion follows directly from the above inequality.
\section{Formal representation of the neural update operator}\label{app:connection}
This appendix gives a formal representation of the neural update operator $F$ in~\eqref{deepspoc-particle-system-batchbybatch} for a gradient-descent update, illustrating how an update of the neural-density parameters induces a corresponding update of the density function itself. 
Since each iteration in the training process depends on the choice of neural network and optimizer, it is difficult to write the process in a fully general form. Although we use Adam as the optimizer in practice, we consider gradient descent here for simplicity. For the same reason, we only consider the loss function $L$ in~\eqref{path loss}. When the activation function in DeepSPoC is chosen as $\tanh$, $\text{softplus}$, or another differentiable function, the generative model $\rho_{\text{NF},\theta}$ and the loss $L$ are also differentiable with respect to $\theta$. Therefore, gradient descent leads to the following update of the parameter $\theta$:
$$
\theta_{p+1}=\left.\theta_{p}-\alpha_p\nabla_{\theta}L(\theta)\right|_{\theta=\theta_p}
$$
where $\alpha_p$ is the learning rate at the $p$-th iteration and 
$$
\left.\nabla_{\theta}L(\theta)\right|_{\theta=\theta_p}=-\sum_{m=0}^{M}\frac{1}{K}\sum_{i=1}^{K}\frac{\nabla_{\theta}\rho_{\text{NF}}(t_m,X_{t_m}^i,\theta_p)}{\rho_{\text{NF}}(t_m,X_{t_m}^i,\theta_p)}.
$$
Note that $\left.\nabla_{\theta}L(\theta)\right|_{\theta=\theta_p}$ is a path/data-dependent vector that stores information from the training paths/data. After the gradient descent step, the generative model $\rho_{\text{NF},\theta}$ is also updated as
$$
\rho_{\text{NF}}(t,x,\theta_{p+1})=\rho_{\text{NF}}\left(t,x,\left.\theta_{p}-\alpha_p\nabla_{\theta}L(\theta)\right|_{\theta=\theta_p}\right).
$$
Since $\rho_{\text{NF},\theta}$ is differentiable with respect to $\theta$, the mean value theorem implies that, for any $(t,x)$, there exists $\xi=\theta_p+(\mu-1)\alpha\left.\nabla_{\theta}L(\theta)\right|_{\theta=\theta_p}$, where $0<\mu< 1$, such that
\begin{equation}\label{iterative_formula_deepspoc}
\rho_{\text{NF}}(t,x,\theta_{p+1})=\rho_{\text{NF}}(t,x,\theta_p)-\alpha_p\left.\nabla_{\theta}^T\rho_{\text{NF}}(t,x,\theta)\right|_{\theta=\xi}\left.\nabla_{\theta}L(\theta)\right|_{\theta=\theta_p}.
\end{equation}
In this case, the operator $F$ in iteration~\eqref{deepspoc-particle-system-batchbybatch} is defined by the right-hand side of~\eqref{iterative_formula_deepspoc}. Since DeepSPoC is a data-based algorithm, the activation function $\text{relu}$ is often used in practice. However, if $\text{relu}$ is used as the activation function, the generative model $\rho_{\text{NF},\theta}$ and the loss $L$ are no longer differentiable with respect to $\theta$. In this case, subdifferentiation must be introduced. The subdifferential version of gradient descent is as follows: suppose $g_p\in \partial_{\theta}L(\theta_p)$, then
$$
\theta_{p+1}=\theta_{p}-\alpha_p g_p.
$$
In this case, the subgradient $g_p$ contains information from the paths/data. If we further assume that $\rho_{\text{NF},\theta}$ is locally Lipschitz continuous, then the Clarke mean value theorem can be applied. Therefore, for any $(t,x)$, there exist $G\in \partial \rho(t,x,\theta)(\xi)$ and $\xi=\theta_p+(\mu-1)g_p$, where $0<\mu<1$, such that
\begin{equation}\label{sub_iterative_formula_deepspoc}
\rho_{\text{NF}}(t,x,\theta_{p+1})=\rho_{\text{NF}}(t,x,\theta_p)-\alpha_p G^Tg_p
\end{equation}
from which the operator $F$ in iteration~\eqref{deepspoc-particle-system-batchbybatch} can be written down. The analysis for SGD and the Adam optimizer is similar.

\bibliographystyle{unsrt}   
\bibliography{spocref} 
\end{document}